\documentclass{article}

\usepackage[final]{neurips_2022}

\usepackage[utf8]{inputenc}
\usepackage[T1]{fontenc}    
\usepackage{hyperref}       
\usepackage{url}            
\usepackage{booktabs}       
\usepackage{amsfonts,mathrsfs}
\usepackage{nicefrac} 
\usepackage{microtype} 
\usepackage{xcolor}

\usepackage{selectp}

\usepackage[utf8]{inputenc}
\usepackage{amsmath}
\usepackage{bm,color}
\usepackage{amssymb}
\usepackage{amsthm}
\usepackage{multicol}
\usepackage{multirow}
\usepackage{mathtools}
\usepackage{mathrsfs} 
\usepackage[capitalise, nameinlink]{cleveref}
\usepackage[normalem]{ulem}
\usepackage{algorithmic}
\usepackage[ruled]{algorithm2e}
\usepackage{subfigure}
\usepackage{graphicx}
\usepackage{pifont}
\usepackage{threeparttable}
\usepackage{microtype}
\usepackage{graphicx}
\usepackage{booktabs}
\usepackage{amsbsy}
\usepackage[normalem]{ulem}
\usepackage{enumitem}
\usepackage{natbib}
\usepackage[pdf]{graphviz}

\usepackage{crossreftools}
\providecommand{\realnum}{\mathbb{R}}

\newtheorem{theorem}{Theorem}
\newtheorem{definition}{Definition}
\newtheorem{lemma}{Lemma}

\newtheorem{assumption}{Assumption}
\theoremstyle{plain}

\theoremstyle{definition}
\theoremstyle{remark}

\title{Robustness in deep learning: The good (width), \\ the bad (depth), and the ugly (initialization)}

\author{Zhenyu Zhu, \quad Fanghui Liu, \quad Grigorios G Chrysos, \quad Volkan Cevher\vspace{2mm} \\
{\hspace*{\fill}EPFL, Switzerland\hspace*{\fill}}\\
{\hspace*{\fill}\texttt{\{[first name].[surname]\}@epfl.ch}\hspace*{\fill}}
}

\begin{document}

\maketitle

\begin{abstract}
We study the average robustness notion in deep neural networks in (selected) wide and narrow, deep and shallow, as well as lazy and non-lazy training settings. We prove that in the under-parameterized setting, width has a negative effect while it improves robustness in the over-parameterized setting. The effect of depth closely depends on the initialization and the training mode. In particular, when initialized with LeCun initialization, depth helps robustness with the lazy training regime. In contrast, when initialized with Neural Tangent Kernel (NTK) and He-initialization, depth hurts the robustness. Moreover, under the non-lazy training regime, we demonstrate how the width of a two-layer ReLU network benefits robustness. Our theoretical developments improve the results by~\citet{huang2021exploring, wu2021wider} and are consistent with~\citet{bubeck2021a, pmlr-v134-bubeck21a}.
\end{abstract}

\section{Introduction}
\label{sec:intro}

It is now well-known that deep neural networks (DNNs) are susceptible to adversarially chosen, albeit imperceptible, perturbations to their inputs~\citep{goodfellow2015explaining, szegedy2013intriguing}.
This lack of robustness is worrying as DNNs are now deployed in many real-world applications~\citep{eykholt2018robust}. As a result,  new algorithms are more and more  being developed to defend against adversarial attacks to improve the DNN robustness. Among the current defense methods, the most commonly used and arguably the most successful method is adversarial training based minimax optimization~\citep{1802.00420, croce2020reliable,madry2019deep}.
To study adversarial attacks and defenses, we need to investigate the robustness of DNNs at first.

A plethora of aspects on the robustness have been studied, ranging from algorithms to their initialization as well as from the width of neural networks to their depth (i.e., the architecture). 
On the practical side, \citet{madry2019deep} advocate that adversarial training requires more parameters (e.g., width) for better performance in minimax optimization, which would fall into the so-called over-parameterized regime\footnote{Over-parameterized regime requires the number of parameters in DNN to be (much) larger than the number of training data.} \citep{zhang2021understanding}.
On the theoretical side, recent works suggest that over-parameterization may damage the adversarial robustness~\citep{huang2021exploring, wu2021wider, zhou2019analysis, https://doi.org/10.48550/arxiv.2201.05149}. In stark contrast, \citet{bubeck2021a, pmlr-v134-bubeck21a} argue that the robustness of DNNs needs enough parameters to be guaranteed. See a detailed discussion in~\cref{sec:related}.

Our work aims to investigate this apparent contradiction in theory, and to close the gap as much as possible. 
We begin with a definition of the \emph{perturbation stability} of DNNs, which can be used to describe the robustness, following the spirit of \citet{wu2021wider, https://doi.org/10.48550/arxiv.2203.11864}. 
\begin{definition}[\emph{perturbation stability}]\label{def:pert}
The perturbation stability of a neural network $\bm{f}(\bm x; \bm W): \mathbb{R}^d \mapsto \mathbb{R}^o$ parameterized by the neural network parameter $\bm W$ under the data distribution $\mathcal{D}_X$ and a perturbation radius $\epsilon$ is defined as follows\footnote{The definition of perturbation stability slightly differs in the lazy and non-lazy training regime of DNNs. Here the lazy/non-lazy training regime indicates that neural network parameters change little/much during training. These two phases are determined by different initializations~\citep{woodworth2020kernel, JMLR:v22:20-1123}. We will detail this in our main results.}:
\begin{equation}
\mathscr{P}(\bm{f}, \epsilon) = \mathbb{E}_{\bm{x},\hat{\bm{x}}, \bm W} \left \| \nabla_{\bm{x}} \bm{f}(\bm{x}; \bm W)^{\top}(\bm{x}-\hat{\bm{x}}) \right \|_2\,,\quad \forall \bm{x} \sim \mathcal{D}_X,~~ \hat{\bm{x}} \sim \text{Unif}(\mathbb{B} (\epsilon, \bm{x}))\,,
\end{equation}
where $\hat{\bm{x}}$ is uniformly sampled from an $\ell_2$ norm ball of $\bm x$ with radius $\epsilon$, denoted as $\text{Unif}(\mathbb{B} (\epsilon, \bm{x}))$. 
\end{definition}
Since our definition of the \emph{perturbation stability} takes the expectation of the clean and the adversarial data points, it is natural to describe the \emph{average robustness} of a neural network. It can be noticed that the larger value of the \emph{perturbation stability} means worse robustness in average, i.e., \emph{average robustness}.

\begin{figure}[t]
\centering
    \includegraphics[width=0.8\linewidth]{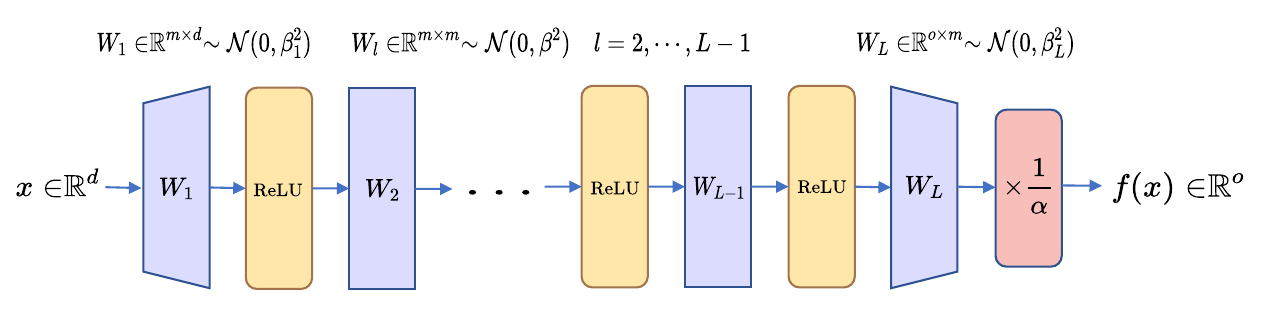}
\caption{Schematic of our deep fully connected ReLU neural network.}\vspace{-0.15cm}
\label{fig:network}
\end{figure}

\begin{table}[t]
\small
\caption{Comparison of the \emph{perturbation stability} of a deep ReLU neural network (see~\cref{fig:network}) under three common Gaussian initializations with different variances.
For a formal definition of this neural network please refer to \cref{eq:deep_network}.}
\label{table:compare_initialization}
\centering
\begin{tabular}{c | c | c }
\toprule 
Initialization name & Initialization form & Our bound for $\mathscr{P}(\bm{f}, \epsilon)/\epsilon$ \\
\midrule
~\citet{lecun2012efficient} & $\beta_1 = \sqrt{\frac{1}{d}}$, $\beta = \beta_L = \sqrt{\frac{1}{m}}$, $\alpha = 1$ & $\bigg(\sqrt{\frac{L^3m}{d}}e^{-m/L^3} + \sqrt{\frac{o}{d}} \bigg)(\frac{\sqrt{2}}{2})^{L-2}$\\
\midrule
~\citet{he2015delving} & $\beta_1 = \sqrt{\frac{2}{d}}$, $\beta = \beta_L =  \sqrt{\frac{2}{m}}$, $\alpha = 1$ & $\sqrt{\frac{L^3m}{d}}e^{-m/L^3} + \sqrt{\frac{o}{d}}$\\
\midrule
~\citet{allen2019convergence} & $\beta = \beta_1 =  \sqrt{\frac{2}{m}}$, $\beta_L = \sqrt{\frac{1}{o}}$, $\alpha = 1$ & $\sqrt{\frac{L^3m}{o}}e^{-m/L^3} + 1 $\\
\midrule
\end{tabular}
\end{table}

Based on the \emph{perturbation stability}, we study the \emph{average robustness} of neural networks under different initializations in (selected) wide and narrow, deep and shallow, as well as lazy and non-lazy training settings.
Generally, non-lazy training makes the analysis of neural networks intractable as DNNs in this regime cannot be simplified as a time-independent model~\citep{chizat2018global}, and accordingly, the analysis in this regime is mainly restricted to the two-layer setting~\citep{1804.06561, mei2019mean}.

Overall, our results suggest that the width (\emph{good}) helps \emph{average robustness} in the over-parameterized regime but the depth (\emph{bad}) can help only under certain initializations (\emph{ugly}). To be specific, we make the following contributions and findings under the lazy/non-lazy training regimes, see Table~\ref{table:compare_initialization}.

In the {\bf lazy training} regime, the derived upper-bounds for DNNs, suggest that 
\begin{itemize}
    \item along with the increase in width, the robustness
    firstly becomes worse in the under-parameterized regime and then gets better, and finally tends to be a constant in highly over-parameterized regimes, which implies the existence of a phase transition.
    \item the depth has more complex tendency on robustness, which largely depends on the initialization and the training mode. It can be grouped into two main classes (\emph{cf.}, Table~\ref{table:compare_initialization}): depth helps robustness in an exponential order under the LeCun initialization~\citep{lecun2012efficient}, whereas it hurts robustness in a polynomial order under He-initialization~\citep{he2015delving} and under Neural Tangent Kernel (NTK) initialization~\citep{allen2019convergence}.
\end{itemize}

Surprisingly, standard tools on training dynamics of neural networks \citep{allen2019convergence,du2018algorithmic} are sufficient to obtain our bounds, which explain the relationship between robustness and the structural/architectural parameters of neural network. Our theoretical developments improve the results by~\citet{huang2021exploring, wu2021wider}, and are supported by empirical evidence. 

In the {\bf non-lazy} training regime, we derive upper-bounds for two-layer networks, suggesting that 
\begin{itemize}
    \item the width improves the robustness under different initializations.
    \item the convergence rates of the average robustness is affected by the initialization. 
\end{itemize}
We also derive a sufficient condition to identify when DNNs enter in this regime, as an initial but first attempt on understanding DNNs in this regime. Our technical contribution lies in connecting robustness to changes of neural network parameters during the early stages of training, which could expand the application scope of deep learning theory beyond \emph{lazy training} analysis \citep{jacot2018neural, allen2019convergence}.

{\bf Notations:} 
We use the shorthand $ [n]:= \{1,2,\dots, n \}$ for a positive integer $n$. 
We denote by $a(n) \lesssim b(n)$: there exists a positive constant $c$ independent of $n$ such that $a(n) \leqslant c b(n)$. The standard Gaussian distribution is $\mathcal{N}(0, 1)$ with the zero-mean and the identity variance. Uniform distribution inside the sphere is $\text{Unif}(\mathbb{B} (\epsilon, \bm{x}))$ with the center $\bm{x}$ and radius $\epsilon$. We follow the standard Bachmann–Landau notation in complexity theory e.g., $\mathcal{O}$, $o$, $\Omega$, and $\Theta$ for order notation.
\section{Related work}
\label{sec:related}

DNNs are demonstrated to be sensitive to adversarially chosen but undetectable noise both empirically \citep{szegedy2013intriguing} and theoretically~\citep{huang2021exploring, bubeck2021a}. Adversarial training~\citep{1802.00420, croce2020reliable,zhang2020over} is a reliable way to obtain adversarially robust neural network. Nevertheless, improving the overall robustness of neural networks is still an unsolved problem in machine learning, especially when coupling with initializations and parameters.

{\bf Over-parameterized neural networks under lazy/non-lazy training regimes:} Modern DNNs in practice~\citep{he2016deep} work under the setting where the number of parameters is (much) larger than the number of training data. Analysis of DNNs in terms of optimization~\citep{safran2021effects, zhou2021local} and generalization~\citep{cao2019generalization} has received great attention in deep learning theory~\citep{zhang2021understanding}.

In deep learning theory, neural tangent kernel (NTK) \citep{jacot2018neural} and mean field~\citep{1804.06561} analysis are two powerful tools for neural network analysis. To be specific, NTK builds an equivalence between training dynamics by gradient-based algorithms of DNNs and kernel regression under a specific initialization, and thus allows the analysis of deep networks~\citep{allen2019convergence, du2019gradient, chen2020much}. However, the NTK requires neural networks to belong in the \emph{lazy training} regime~\citep{chizat2019lazy}, where neural networks are able to achieve zero training loss but the parameters change little, or even remain unchanged during training. In contrast, mean-field theory establishes global convergence by casting the network weights during training as an evolution in the space of probability distributions under some certain initializations~\citep{1804.06561,chizat2018global}. This strategy goes beyond lazy training regime, which allows for neural networks parameters to change in a constant order after training.

If the neural networks parameters change a lot after training, or even tend to infinity, then neural networks work in the \emph{non-lazy} training regime. Analysis of DNNs under this setting appears intractable and challenging, so the current work mainly focus as two-layer neural networks~\citep{1803.08367,JMLR:v22:20-1123}.

{\bf Robustness and over-parameterization} 
\citet{goodfellow2015explaining} demonstrate that adversarial learning helps robustness and reduces overfitting. Many works focus on influencing factors of adversarial examples and robustness of the neural network~\citep{schmidt2018adversarially, zhang2020understanding, allen2022feature}. The relation between model capacity and robustness is empirically investigated by~\citet{madry2019deep}, i.e., a neural network with insufficient capacity can seriously hurt the robustness.~\citet{pmlr-v134-bubeck21a} theoretically study the inherent trade-off between the size of neural networks and their robustness, and they claim that over-parameterization is necessary for the robustness of two-layer neural networks.

However, some recent works propose the opposite view. Under the lazy training regime,~\citet{huang2021exploring} demonstrate that when over-parameterized neural networks get wider, the robustness decreases in a polynomial order. Similarly, the depth hurts the robustness in an exponential order. ~\citet{wu2021wider} affirm the view of~\citet{huang2021exploring} on the width. However for depth, they derive a stronger bound that the robustness gets worse in a polynomial decay as the depth increases, as suggested by \citet{https://doi.org/10.48550/arxiv.2201.05149}: over-parameterization hurts robustness. In addition,~\citet{gao2019convergence} also make a similar view: an increased model capacity (i.e., wider width and deeper depth) deteriorates the robustness of neural networks. Nevertheless, we remark that, the results of \citet{https://doi.org/10.48550/arxiv.2201.05149} work in a slightly different setting than that of~\citet{pmlr-v134-bubeck21a} on data interpolation, which requires a careful comparison. Accordingly, we adopt a complementary view to the vast literature. We provide an in-depth theoretical analysis to investigate this apparent contradiction in theory, and to close the gap as much as possible.
\section{Problem setting}
\label{sec:problem_setting}

Let $X \subseteq \mathbb{R}^d$ and $Y \subseteq \mathbb{R}^{o}$ be compact metric spaces. We assume that the training set $\mathcal{D}_{tr} = \{  (\bm x_i, \bm y_i) \}_{i=1}^n $ is drawn from a unknown probability measure $\mathcal{D}$ on $X \times Y$. Its marginal data distribution is denoted by $\mathcal{D}_X$. The goal of the classification task is to learn a neural network  $\bm{f}: X \rightarrow Y$ such that $\bm{f}(\bm x; \bm W)$ parameterized by $\bm W$ is a good approximation of the label $\bm{y} \in Y$ corresponding to a new sample $\bm x \in X$. In this paper, we use the empirical risk $L(\bm{W}) = \frac{1}{2n}\sum_{i=1}^{n}\left \| \bm{f}(\bm x_i; \bm W) - \bm{y}_i \right \|_2^2$. Then we make the following assumption.

\begin{assumption}
\label{assumption:distribution_1}
We assume that the data satisfy $\| \bm x \|_{2} = 1$.
\end{assumption}

{\bf Remark:} This assumption is standard in theory of over-parameterized neural networks and also commonly used in practice~\citep{du2018gradient, du2019gradient, allen2019convergence, oymak2020toward, pmlr-v119-malach20a}. 

\subsection{Network}
\label{ssec:network}

We focus on the typical depth-$L$ fully-connected ReLU neural networks admitting the width $m_l$ of the $l$-th hidden layer, $\forall l \in [L]$  (\emph{cf.}, \cref{fig:network}):

\begin{equation}
    \bm{h}_{i,0} = \bm{x}_i; \quad \bm{h}_{i,l} = \phi(\bm{W}_l\bm{h}_{i,l-1}); \quad \bm{f}(\bm x_i; \bm W) = \bm{f}_{i} = \frac{1}{\alpha}\bm{W}_L\bm{h}_{i,L-1}; \quad \forall l\in[L-1]\,,~ i\in[n]\,,
\label{eq:deep_network}
\end{equation}

where $\bm{x}\in \mathbb{R}^d$, $\bm{f}(\bm{x}) \in \mathbb{R}^o$, $\alpha$ is the scaling factor, and $\phi = \max(0,x)$ is the ReLU activation function. The neural network parameters formulate the tuple of weight matrices $\bm{W} := \{ \bm{W}_i \}_{i=1}^L \in  \{ \mathbb{R}^{m\times d} \times (\mathbb{R}^{m\times m})^{L-2}\times \mathbb{R}^{o \times m} \}$. According to the property $\phi(x) = x\phi^{\prime}(x)$ of ReLU, we have $\bm{h}_{i,l} = \bm{D}_{i,l}\bm{W}_l\bm{h}_{i,l-1}$, where $\bm{D}_{i,l}$ is a diagonal matrix under the ReLU activation function defined as below.
\begin{definition}[Diagonal sign matrix]
\label{def:diagonal_sign_matrix}
For each $i \in [n]$, $l\in [L-1]$ and $k \in [m]$, the diagonal sign matrix $\bm{D}_{i,l}$ is defined as: $(\bm{D}_{i,l})_{k,k} = 1\left \{ (\bm{W}_l\bm{h}_{i,l-1})_k \geq 0 \right \} $.
\end{definition}

In our setting, we consider the standard Gaussian initialization with different variances that includes three typical initialization schemes in practice.

{\bf Initialization:} Let $m_0 = d$, $m_L = o$ and $m_2 = \cdots = m_{L-1} = m$, we make the standard random initialization $[\bm{W}_l]_{i,j}\sim \mathcal{N}(0,\beta_l^2)$ for every $(i, j) \in [m_l] \times [m_{l-1}]$ and $l \in [L]$.
Choosing a different variance, our work holds for three commonly used Gaussian initializations, i.e., LeCun initialization~\citep{lecun2012efficient}, He-initialization~\citep{he2015delving} and Neural Tangent Kernel (NTK) initialization~\citep{allen2019convergence}, refer to the formal definition in \cref{table:compare_initialization} for details.

\subsection{Discussion on various robustness metrics} 

In~\cref{sec:intro}, we have proposed our robustness metric: perturbation stability (\emph{cf.}, \cref{def:pert}). This metric can be viewed as an expectation of the inner product of first-order approximation of adversarial risk~\citep{madry2019deep} and the perturbations with the uniform distribution, which measures the \emph{average robustness} of the neural network. As we mentioned before, under the same perturbation radius $\epsilon$, a smaller value of $\mathcal{P}(\bm{f}, \epsilon)$ implies better \emph{perturbation stability}, that is better \emph{average robustness}. Previous works~\citep{hein2017formal, weng2018evaluating, wu2021wider, bubeck2021a} use Lipschitzness to describe the robustness of the network, suggesting that smaller Lipschitzness leads to robust models. However, Lipschitzness is only a worst-case measure, and cannot reasonably describe the average changes of the entire dataset. Instead, we follow the measure of~\citet{wu2021wider, https://doi.org/10.48550/arxiv.2203.11864}, that aims to comprehensively consider the overall distribution of the data, not only the extreme case.
Besides, the worst-case robustness can be extended to a probabilistic robustness view~\citep{robey2022probabilistically}, which shares a similar spirit as our average robustness concept.
\citet{schmidt2018adversarially} present another definition of robustness, depending on the misclassified error of an adversarial data point.
Instead, our perturbation stability measures the function value changes at the clean data point via Taylor expansion which can exclude the influence of the learning capacity of the network.
\section{Main results}

In this section, we state the main theoretical results. Firstly, in~\cref{ssec:bound_lazy} we provide the upper bound of the \emph{perturbation stability} in lazy training regime for deep neural networks defined by~\cref{eq:deep_network}. The sufficient condition that the neural network~\cref{eq:deep_network} is under non-lazy training regime is given in~\cref{ssec:Sufficient_condition_non_lazy}. Finally, in~\cref{ssec:bound_non_lazy}, we provide the upper bound on the \emph{perturbation stability} during early training of two-layer network under the non-lazy training regime. 

\subsection{Upper bound of the perturbation stability of DNNs under the lazy training regime} 
\label{ssec:bound_lazy}

We are now ready to state the main results under the lazy training regime. The following theorem provides the upper bound of the \emph{perturbation stability} and connects 
to the width, and the depth of a deep fully-connected neural network under different standard Gaussian initializations.
The proof of~\cref{thm:perturbation_stability_lazy} is deferred to~\cref{sec:proofs_thm1}.

\begin{theorem}
\label{thm:perturbation_stability_lazy}
Given an $L$-layer neural network $\bm{f}$ defined by~\cref{eq:deep_network} trained by $\left \{ (\bm{x}_i, \bm{y}_i) \right \} _{i=1}^n$ satisfying~\cref{assumption:distribution_1}, for the convenience of analysis, we set $\alpha = 1$ and $\beta : = \beta_2 = \dots = \beta_{L-1}$, define a constant $\gamma := \beta/\sqrt{\frac{2}{m}}$. Then, under a small perturbation $\epsilon$, we have the following:
\begin{equation}
    \frac{\mathscr{P}(\bm{f}, \epsilon)}{\epsilon} \lesssim \bigg(\sqrt{L^3m^2\beta_1^2\beta_L^2}e^{-m/L^3} + \sqrt{m o \beta_1^2 \beta_L^2} \bigg)\gamma^{L-2} \,.
\label{eq:thm_1}
\end{equation}
\end{theorem}

{\bf Remark:} Our results cover the effect of the width and depth of neural network on robustness under various common initializations depending on $\gamma \gtreqqless 1$.

{\bf Comparison with three commonly used initializations:}
For the initializations used in practice, our theoretical results can be mainly divided into two classes: 1) The depth helps robustness in an exponential order under the LeCun initialization: \cref{thm:perturbation_stability_lazy} implies that $\bigg(\sqrt{\frac{ L^3m}{d}}e^{-m/L^3} + \sqrt{\frac{o}{d}} \bigg)(\frac{\sqrt{2}}{2})^{L-2}$. 2) The depth hurts
the robustness in a polynomial order under He-initialization $\bigg(\sqrt{\frac{ L^3m}{d}}e^{-m/L^3} + \sqrt{\frac{o}{d}}\bigg)$ and under the NTK initialization $\bigg(\sqrt{\frac{ L^3m}{o}}e^{-m/L^3} + 1\bigg)$ derived by \cref{thm:perturbation_stability_lazy}. 
When employing other initializations, the robustness could be hurted in a exponential order. Below, we elaborate on these three initalizations:

    {\bf 1) LeCun initialization} ($\gamma=\frac{\sqrt{2}}{2}$): The order has three main parts: $\sqrt{\frac{L^3m}{d}}$, $e^{-m/L^3}$ and $(\frac{\sqrt{2}}{2})^{L-2}$.  
    Regarding the width $m$, the first part $\sqrt{\frac{ L^3m}{d}}$ is an increasing function of $m$ and the second part $e^{-m/L^3}$ is a decreasing function of $m$. 
    Accordingly, in the under-parameterized region (e.g., $m$ is small), $\sqrt{\frac{L^3 m}{d}}$ plays a major role, so the stability will increase as $m$ increases. After a critical point, $e^{-m/L^3}$ plays a major role, so the stability will decrease as $m$ increases. When $m$ tends to infinity, the first term of the bound tends to $0$. Hence the \emph{perturbation stability} tends to be a constant and independent of the width $m$ as the width $m$ tends to infinity.
    It means that there exists a phase transition phenomenon between the \emph{perturbation stability} and over-parameterization in terms of the width $m$.
    
    Regarding the depth $L$, the constant $\gamma = \frac{\sqrt{2}}{2}$ implies that the third part has a faster decrease speed than the first and second parts and plays a major role in the tendency. The \emph{perturbation stability} of the neural network exponentially decreases with respect to the depth.
    That means, for the LeCun initialization, the deeper the network, the better the robustness. Nevertheless, the energy of the LeCun initialization decreases as the network depth increases due to the variance $\beta = \sqrt{\frac{1}{m}}$. Since the activation function ReLU loses half of the energy in every layer, training a deep network with the LeCun initialization is difficult. Hence, we need a trade off between robustness and training difficulty regarding the network depth in practice for the LeCun initialization. 
    
    {\bf 2) He initialization and NTK initialization} ($\gamma=1$): the 
    bounds for these two initializations are almost the same, and only differ in the feature dimension.
    We can see that phase transition phenomena exist under these two initializations regarding the width $m$, similar to the LeCun initialization. 
    Regarding the depth, when $L$ is large, the first part $\sqrt{L^3}$ plays a major role in the
    \emph{perturbation stability}. So these two initializations hurt the robustness of the neural network at a polynomial order.
    
    All of the three initializations admit $\gamma \leq 1$. If some initialization schemes admit $\gamma > 1$, then the depth $L$ will hurt the robustness of the neural network at an exponentially increasing rate.

{\bf Comparison with previous work:}
~\cref{thm:perturbation_stability_lazy} provides a new relationship between the robustness with width and depth of DNNs. We compare our result with
\citep{wu2021wider, huang2021exploring} using a basic NTK initialization~\citep{allen2019convergence} (suppose that $m\gg o$ and $m\gg d$). 
For a better comparison, we derive their results under our robustness metric \emph{perturbation stability}, as reported by~\cref{table:compare_bound}.

\begin{table}[t]
\caption{Comparison of the orders of the proposed bound with other two recent works. Our results are general to cover both under- and over-parameterized regimes, which expands the application scope of previous results~\citep{wu2021wider, huang2021exploring}.
(The original result of ~\citep{wu2021wider} can be reduced to $\sqrt{mL}$ as the $\frac{m}{(\log m)^6} \geq L^{12}$ condition is required).}
\label{table:compare_bound}
\centering
\begin{tabular}{c | c | c | c}
\toprule
Metrics & Our result & \citet{wu2021wider} & \citet{huang2021exploring} \\
\midrule
$\mathscr{P}(\bm{f}, \epsilon) / \epsilon$ & $\sqrt{\frac{L^3m}{o}}e^{-m/L^3} + 1$ & $L^2 m^{1/3}\sqrt{\log m} + \sqrt{mL}$ & $2^{\frac{3L-5}{2}}\sqrt{m}$\\
\midrule
\end{tabular}
\end{table}

Our results indicate a behavior transition on the width. For the over-parameterized regime, the robustness of the neural network only depends on the perturbation energy, and it is almost independent of the width $m$. The results on the width are significantly better than the previous results increasing as the square root of $m$. For depth $L$, our results provide a tighter and more precise estimate as compared to \citep{wu2021wider} in a two-degree polynomial increasing order and \citep{huang2021exploring} in an exponential increasing order.

Furthermore, compared with the results of~\citet{pmlr-v134-bubeck21a} showing that the robustness of the two-layer neural network becomes better with the increase of the number of neurons (i.e., width), we provide a more detailed and refined result on the robustness of DNNs under different initializations and under-parameterized regimes.

\subsection{Sufficient condition for neural network under non-lazy training regime}
\label{ssec:Sufficient_condition_non_lazy}

Beyond the lazy training regime, we turn our attention to the non-lazy training regime and provide a sufficient condition for a (well-chosen) initialization of neural networks when entering into the non-lazy training regime. This is a first attempt to understand the training dynamics of DNNs in this regime.

Our result requires a further assumption on the data and the empirical risk as follows.
\begin{assumption}
\label{assumption:distribution_2}
For a single-output network defined in~\cref{eq:deep_network}, we assume that $\max_{i \in [n]} y_i \geq C_1 > 0$ for some constant $C_1$. 
We also assume that the neural network can be well-trained such that the empirical risk is $\mathcal{O}(\frac{1}{n})$.
\end{assumption}

{\bf Remark:} This is a common assumption in the field of optimization~\citep{song2021subquadratic} in the under- and over-parameterized regime, and we can even assume zero risk. Here we follow the specific assumption of~\citet{JMLR:v22:20-1123}.

Now we are ready to present our result: a sufficient condition to identify when deep ReLU neural networks fall into the non-lazy training regime, as a promising extension of \cite{JMLR:v22:20-1123} on two-layer neural networks.
To avoid cluttering the analysis, we assume a single-output i.e., $o = 1$. The proof of~\cref{thm:sufficient_condition_of_non-lazy_training} is deferred to~\cref{sec:proofs_thm2}.

\begin{theorem}
\label{thm:sufficient_condition_of_non-lazy_training}
Given an $L$-layer neural network $\bm{f}$ defined by~\cref{eq:deep_network} with $o = 1$, trained by $\left \{ (\bm{x}_i, \bm{y}_i) \right \} _{i=1}^n$, under Assumptions~\ref{assumption:distribution_1} and~\ref{assumption:distribution_2}, suppose that  $\alpha \gg (m^{3/2}\sum_{i=1}^{L}\beta_i)^L$ and $m \gg d$ in \cref{eq:deep_network}, then
for sufficiently large $m$, with probability at least $1-(L-2)\exp(-\Theta(m^2))-\exp(-\Theta(md))-\exp(-\Theta(m))$ over the initialization, we have:

\begin{equation*}
    \sup_{t \in [0,+\infty)} \frac{\left \| \bm{W}_l(t)-\bm{W}_l(0) \right \|_{\mathrm{F}} }{\left \| \bm{W}_l(0) \right \|_{\mathrm{F}}} \gg 1\,, \quad \forall l \in [L] \,.
\end{equation*}
\end{theorem}

{\bf Remark:} The condition $\alpha \gg (m^{3/2}\sum_{i=1}^{L}\beta_i)^L$ implies that, a neural network falls in a non-lazy training regime when the variance of the Gaussian initialization is very small.
It can be achieved by a typical case: taking $m \gg L^2$, choosing $\alpha = 1$ and $\forall l \in [L]; \ \beta_l = \frac{1}{m^2}$.
Commonly used initializations such as NTK initialization, LeCun initialization, He's initialization lead to  lazy training.
\subsection{Upper bound of the perturbation stability for two-layer networks in non-lazy training}
\label{ssec:bound_non_lazy}
Unlike lazy training, weights of non-lazy training concentrate on few directions determined by the input data in the early stages of training~\citep{JMLR:v22:20-1123}.
The following theorem describes the neural network \emph{perturbation stability} in the early training stage as a function of network width in the non-lazy training regime.
For ease of description, here we consider a special initialization scheme under the non-lazy regime, the proof of~\cref{thm:perturbation_stability_non_lazy} is deferred to~\cref{sec:proofs_thm3}. 
In the non-lazy training regime, the movement of parameters in DNNs does not convergence to zero. 
In this case, the expectation over $\bm W$ at the initialization in our~\cref{def:pert} does not make sense.
Hence we slightly modify the definition of perturbation stability by dropping out the expectation over $\bm W$.

\begin{definition}[\emph{perturbation stability}]\label{def:pert_}
The perturbation stability of a single output neural network $f(\bm x; \bm W): \mathbb{R}^d \mapsto \mathbb{R}$ parameterized by the neural network parameter $\bm W$ under the data distribution $\mathcal{D}_X$ and a perturbation radius $\epsilon$ is defined as follows:
\begin{equation}
\mathscr{P}(f, \epsilon) = \mathbb{E}_{\bm{x},\hat{\bm{x}}} \left \| \nabla_{\bm{x}} f(\bm{x}; \bm W)^{\top}(\bm{x}-\hat{\bm{x}}) \right \|_2\,,\quad \forall \bm{x} \sim \mathcal{D}_X,~~ \hat{\bm{x}} \sim \text{Unif}(\mathbb{B} (\epsilon, \bm{x}))\,,
\end{equation}
where $\hat{\bm{x}}$ is uniformly sampled from an $\ell_2$ norm ball of $\bm x$ with radius $\epsilon$, denoted as $\text{Unif}(\mathbb{B} (\epsilon, \bm{x}))$. 
\end{definition}
Here the modified perturbation stability is a random variable, which is different from the deterministic version in \cref{def:pert}.
Based on our modified metric, we are ready to present our theorem as below.

\begin{theorem}
\label{thm:perturbation_stability_non_lazy}
Given a two-layer neural network with single output $f_t$ defined by~\cref{eq:deep_network} and trained by $\left \{ (\bm{x}_i, y_i) \right \} _{i=1}^n$ satisfying~\cref{assumption:distribution_1}, using gradient descent under the squared loss, consider the following initialization in~\cref{eq:deep_network}: $L=2$, $\alpha \sim 1$, $\beta_1 \sim \beta_2 \sim \beta \sim \frac{1}{m^c}$ with $c \geq 1.5$, $m\gg n^2$ and training time less than a constant that only depends on $n, m$ and $\lambda_0$, then for a small range of perturbation $\epsilon$, with probability at least $1-n\exp(-\frac{n}{2})-\frac{3}{n}$ over initialization, we have the following:
\begin{equation}
    \frac{\mathscr{P}(f_t, \epsilon)}{\epsilon} \leq \Theta \bigg( \frac{\sqrt{n \log m} + n}{m^{c-1}} \bigg(\frac{1}{\sqrt{n^3m}}+\frac{1}{m^{c-0.5}} \bigg)\bigg)\,.
\label{eq:thm_3}
\end{equation}
\end{theorem}

{\bf Remark:} Under this setting of non-lazy training regime, the robustness and width of the neural network are positively correlated in the early stages of training. That is, as the width $m$ increases in the over-parameterized regime, a Gaussian initialization with smaller variance leads to the robustness increasing in a faster decay. Our result holds for other initialization schemes in the non-lazy training regime, e.g., $c = 2$ leads to $\mathscr{P}(f_t, \epsilon)/\epsilon \leq \Theta \bigg( \frac{\sqrt{n \log m} + n}{m^{2.5}}\bigg)$; and $c = 3$ leads to  $\mathscr{P}(f_t, \epsilon)/\epsilon \leq \Theta \bigg( \frac{\sqrt{n \log m} + n}{m^{4.5}}\bigg)$. 
\section{Numerical evidence}
\label{sec:experiment}

We validate our theoretical results with a series of experiments. In~\cref{ssec:Validation_lazy}  we firstly verify that our initialization settings belong in the lazy and the non-lazy training regimes. In~\cref{ssec:Validation_width}, we explore the effect of varying widths from under-parameterized to over-parameterized regions on the \emph{perturbation stability} of neural networks. In~\cref{ssec:Validation_depth_initialization}, we finally compare the effect of two different initializations and the network depth on the \emph{perturbation stability}. Additional experimental results can be found in \cref{sec:additional_experiments}.

\subsection{Experimental settings}
\label{ssec:experiment_setting}

Here we present our experimental setting including models, hyper-parameters, the choice of width and depth, and initialization schemes.
We use the popular datasets of MNIST~\citep{726791} and CIFAR-10 ~\citep{krizhevsky2014cifar} for experimental validation.

\textbf{Models:} We report results using the following models: fully connected ReLU neural network named ``FCN'' in main paper and convolutional ReLU neural network named ``CNN'' in~\cref{sec:additional_experiments}.

\textbf{Hyper-parameters:} Unless mentioned otherwise, all models are trained for 50 epochs with a batch size of $64$. The initial value of the learning rate is $0.001$. After the first $25$ epochs, the learning rate is multiplied by a factor of $0.1$ every $10$ epochs. The SGD is used to optimize all the models, while the cross-entropy loss is used.

{\bf Width and depth:} In order to verify our theoretical results, we conduct a series of experiments with different depths and widths of the same type neural network. Specifically, our experiments include $11$ different widths from $2^4$ to $2^{14}$, and four different choices of depths, i.e., $2, 4, 6, 8, 10$.

{\bf Initialization:} We report results using the following initializations: 1) He initialization where $W_{ij} \sim \mathcal{N}(0,\frac{2}{m_{in}}$), 2) LeCun initialization where $W_{ij} \sim \mathcal{N}(0,\frac{1}{m_{in}}$) and 3) an initialization that allows for non-lazy training regime on two-layer networks, i.e., $\beta_1 = \beta_2 = 1/m^2$ and $\alpha = 1$.

\subsection{Validation of lazy and non-lazy training regimes}
\label{ssec:Validation_lazy}

Before verifying our results, we need to identify the lazy and non-lazy training regime under different initializations.
To this end, we define a measure for the \emph{lazy training ratio}, i.e., $\kappa =  \frac{\sum_{l=1}^{L} \left \| \bm{W}_l(t)-\bm{W}_l(0) \right \|_{\mathrm{F}} }{\sum_{l=1}^{L} \left \| \bm{W}_l(0) \right \|_{\mathrm{F}}}$. This measure evaluates whether the neural network is under the lazy training regime. 
A smaller $\kappa$ implies that the neural network is close to lazy training. 

According to the theory, we employ the He initialization and the non lazy training initialization we state in~\cref{ssec:experiment_setting} to conduct the experiment under two-layer neural networks to verify that their lazy training ratio matches the theoretical results of lazy training and non-lazy training (i.e., the experiment is under the correct regime). \cref{fig:verify_lazy_time_1} and~\cref{fig:verify_lazy_width_1} show the tendency of ratio with respect to time (training epochs) and relationship between width and lazy training ratio of neural networks under lazy training regime, respectively. We find that the ratio of lazy traing regime is almost a constant that does not change with time, and this constant decreases as the width of the network increases. This is in line with what we know about lazy training~\citep{chizat2019lazy}. 

Likewise, \cref{fig:Validation_for_non_lazy} shows the ratio tendency with time and width under non-lazy training regime. The ratio increases almost linearly over time in~\cref{fig:verify_lazy_time_2}. In epoch $25$ we decrease the learning rate, which decreases the rate that $\kappa$ increases. At the same time,~\cref{fig:verify_lazy_width_2} shows a similar tendency between the width and lazy training ratio as lazy training. 
However, the value of $\kappa$ is much higher than that of lazy training regime. Combining the results about tendency with time, the ratio will be expected to increase as the time until infinity.

\begin{figure}[t]
\centering
    \subfigure[]{\includegraphics[width=0.47\linewidth]{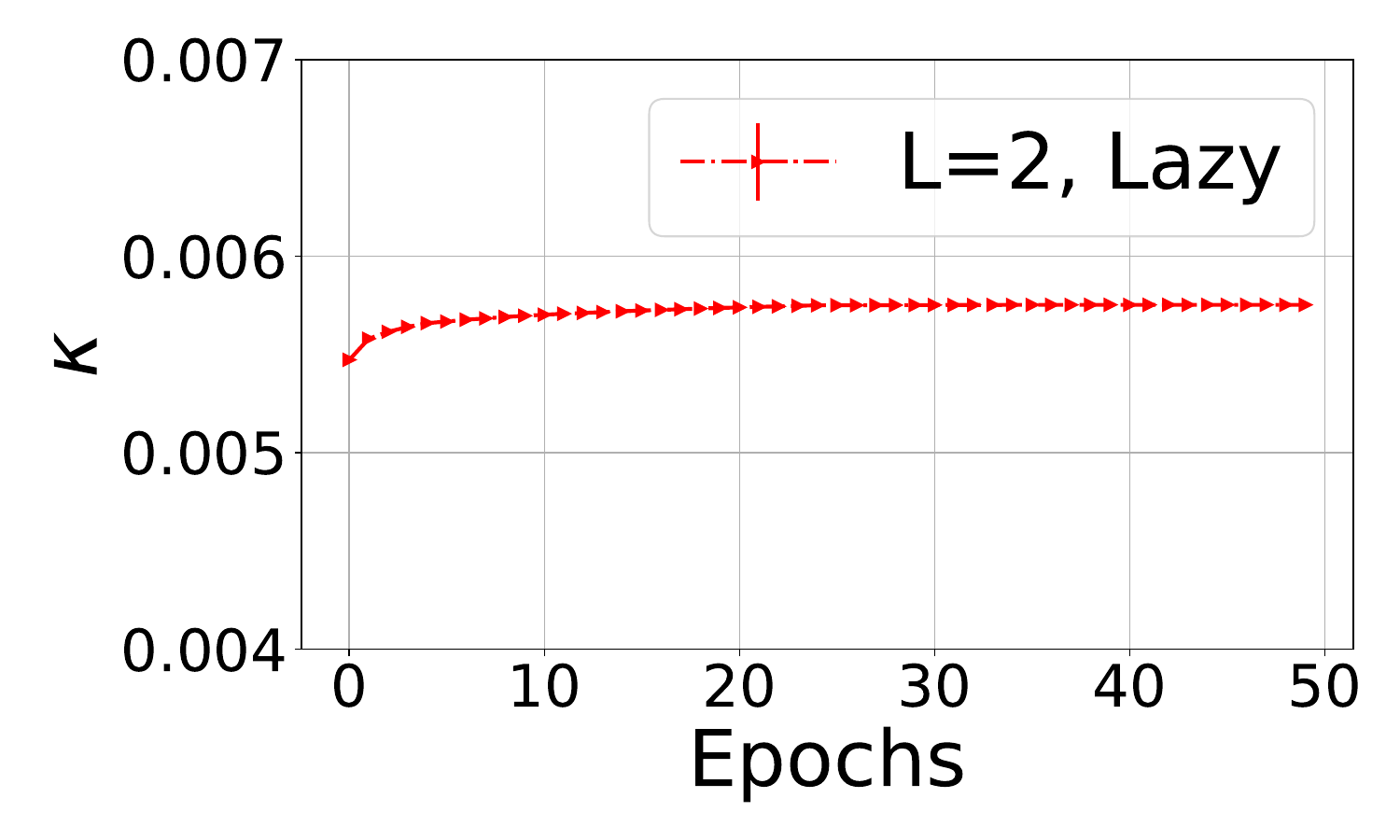}\label{fig:verify_lazy_time_1}}
    \subfigure[]{\includegraphics[width=0.45\linewidth]{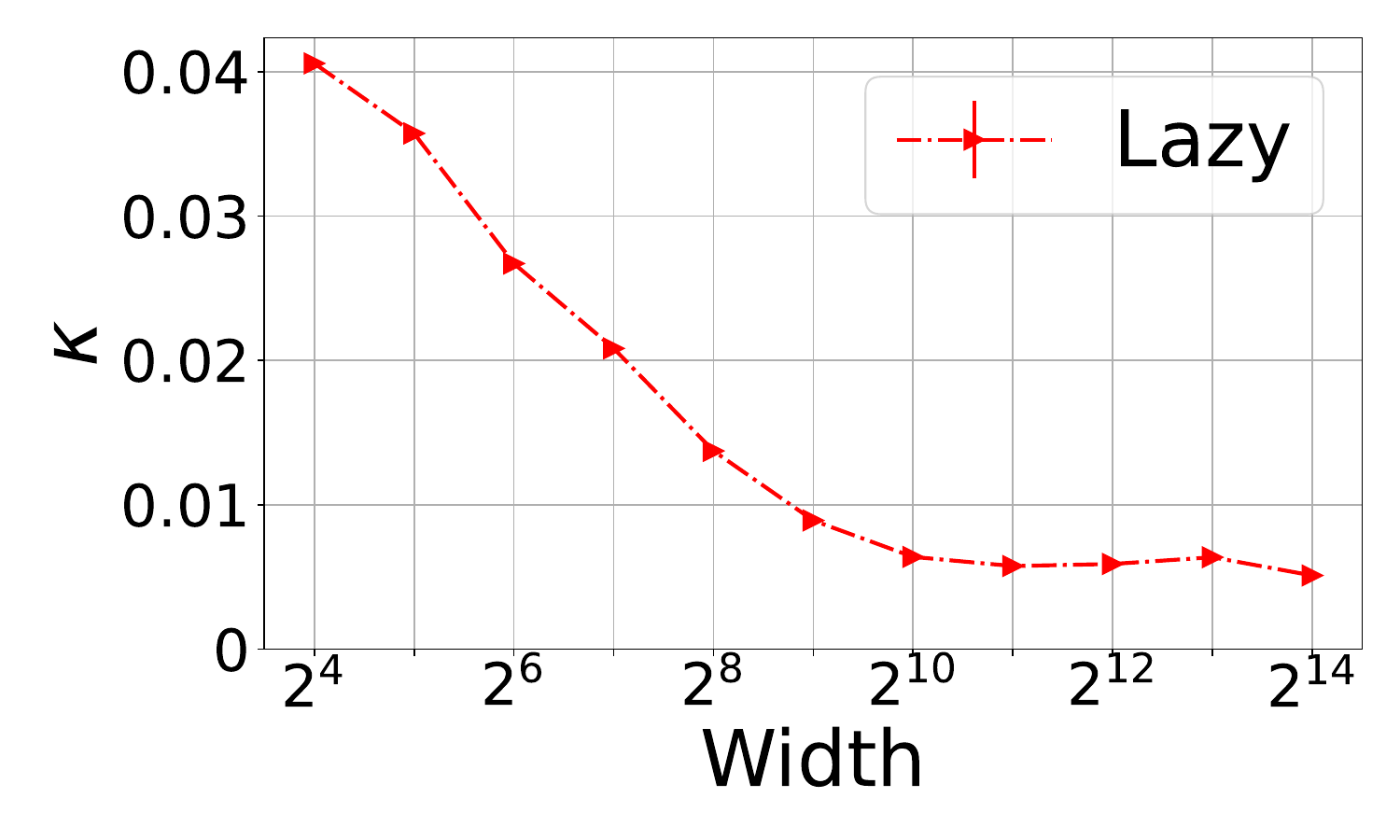}\label{fig:verify_lazy_width_1}}
\caption{(a) Tendency with respect to time (training epochs) and (b) relationship between width and lazy training ratio of neural networks. ~\cref{fig:verify_lazy_time_1} shows that ratio $\kappa$ is small and almost unchanged, recognized as \emph{lazy training}. In~\cref{fig:verify_lazy_width_1}, we can see that the $\kappa$  decreases with the increasing width.}\vspace{-0.15cm}
\label{fig:Validation_for_lazy_non_lazy}
\end{figure}

\begin{figure}[t]
\centering
    \subfigure[]{\includegraphics[width=0.45\linewidth]{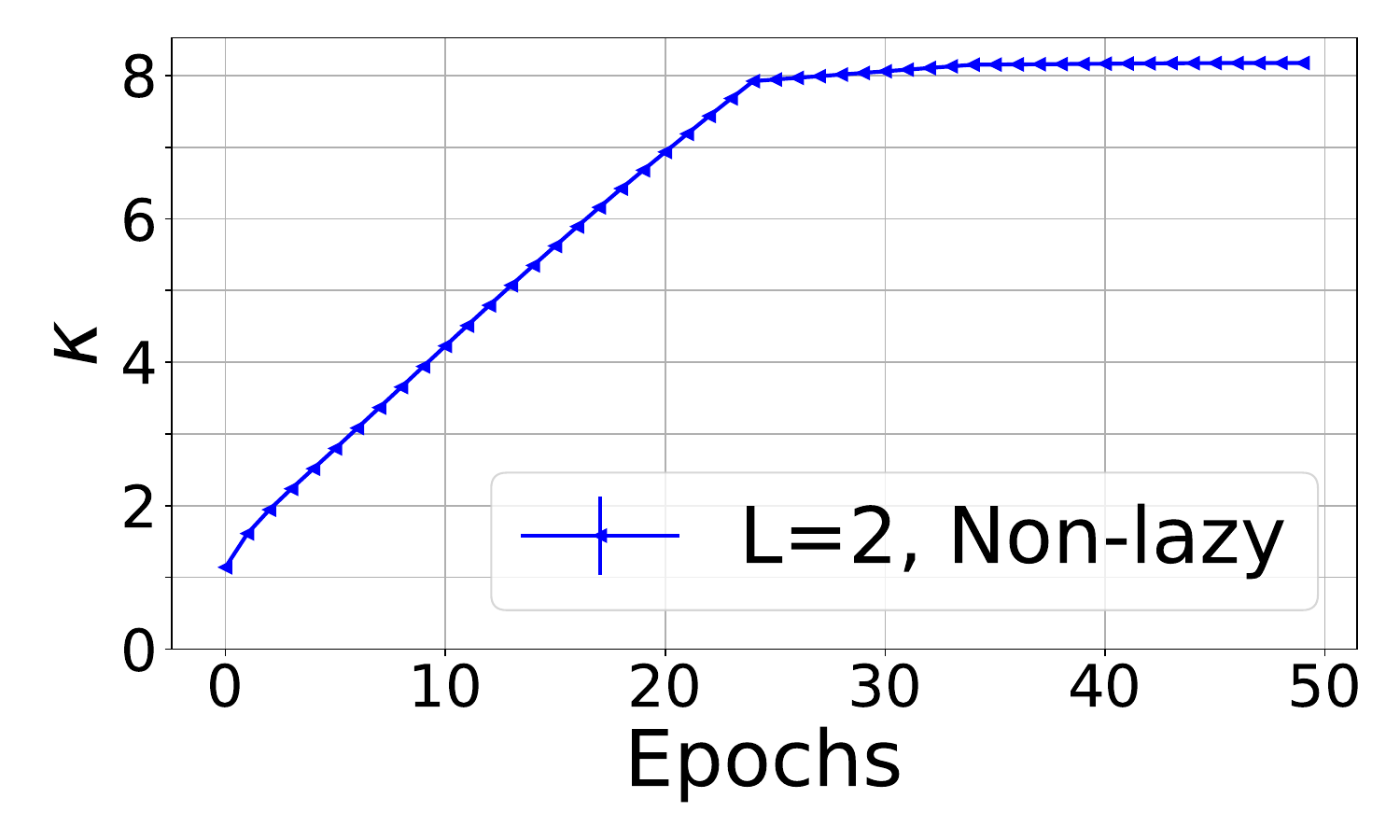}\label{fig:verify_lazy_time_2}}
    \subfigure[]{\includegraphics[width=0.45\linewidth]{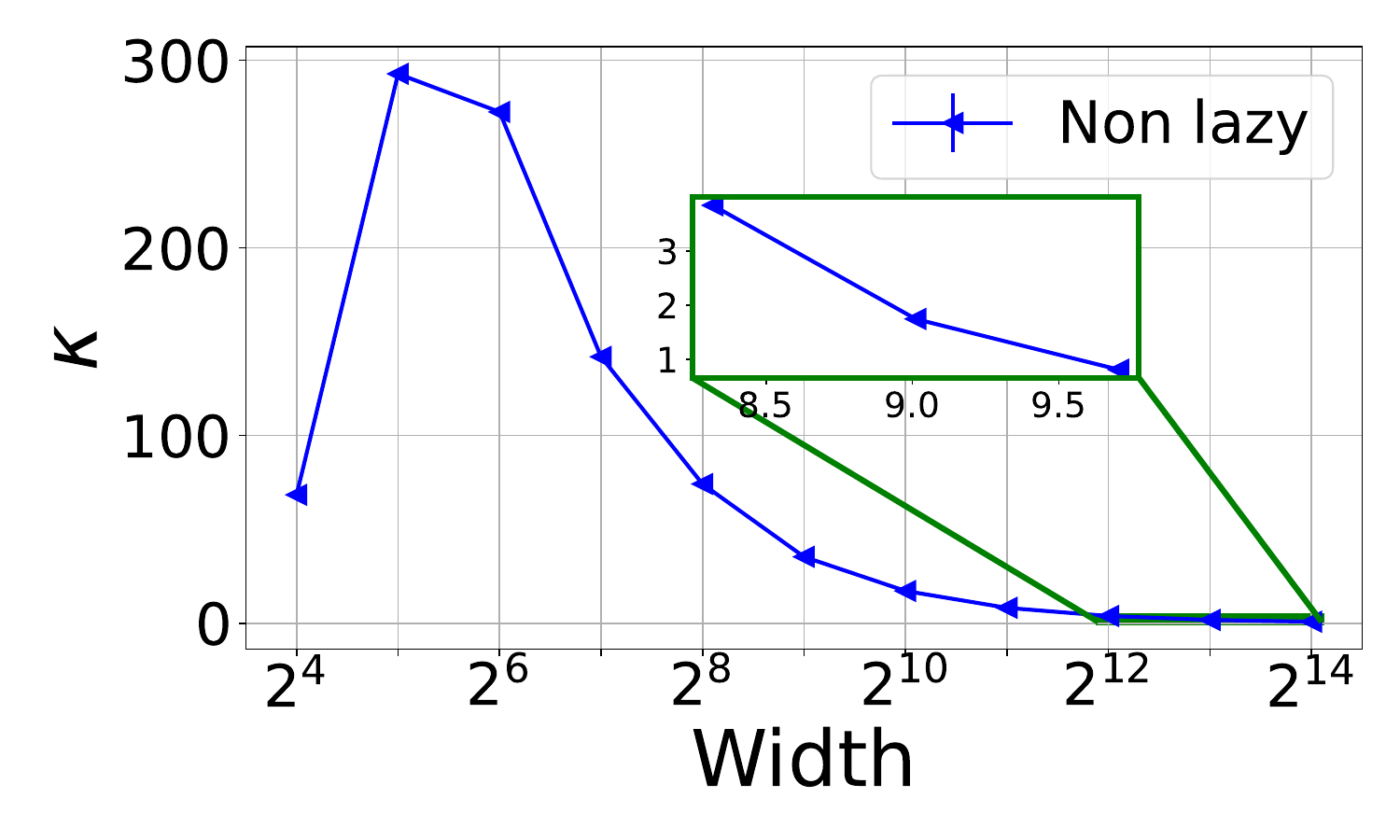}\label{fig:verify_lazy_width_2}}
\caption{(a) Tendency with respect to time (training epochs) and relationship with width of non-lazy training ratio of neural networks. ~\cref{fig:verify_lazy_time_2} shows that ratio $\kappa$ is changed a lot (increasing and then remains unchanged), recognized as \emph{non-lazy training}. The tendency of $\kappa$ for non-lazy training is increasing with the width and then decreasing, i.e., a phase transition in~\cref{fig:verify_lazy_width_2}.}\vspace{-0.15cm}
\label{fig:Validation_for_non_lazy}
\end{figure}

\subsection{Validation for width}
\label{ssec:Validation_width}

We verify the relationship between the \emph{perturbation stability} and the width of network as illustrated by~\cref{eq:thm_1,eq:thm_3}. We conduct a series of experiments on MNIST dataset using FCN with different widths.~\cref{fig:verify_width} shows the relationship between the \emph{perturbation stability} and width of FCN with different depths and training regimes. Here for lazy training and non-lazy training we use the same initialization as~\cref{ssec:Validation_lazy}.

~\cref{fig:verify_width_lazy} exhibits the relationship between the \emph{perturbation stability} and the width of neural networks with different depths for $L=2, 4, 6, 8$, and $10$. All of the five curves confirm the phase transition with width: the \emph{perturbation stability} firstly increases and then decreases with width, which match our theoretical results.
\cref{fig:verify_width_non_lazy} shows the difference of the effect of width on the \emph{perturbation stability} of lazy and non-lazy training for two-layer neural networks. The \emph{perturbation stability} of non-lazy training is significantly smaller than that of lazy training, which means non-lazy training regime is more robust. Besides, the \emph{perturbation stability} of non-lazy training decreases with the width of the neural network increases, which coincides with our theoretical result, i.e., no phase transition phenomenon.

\begin{figure}[t]
\centering
    \subfigure[Lazy training with different widths and depths]{\includegraphics[width=0.49\linewidth]{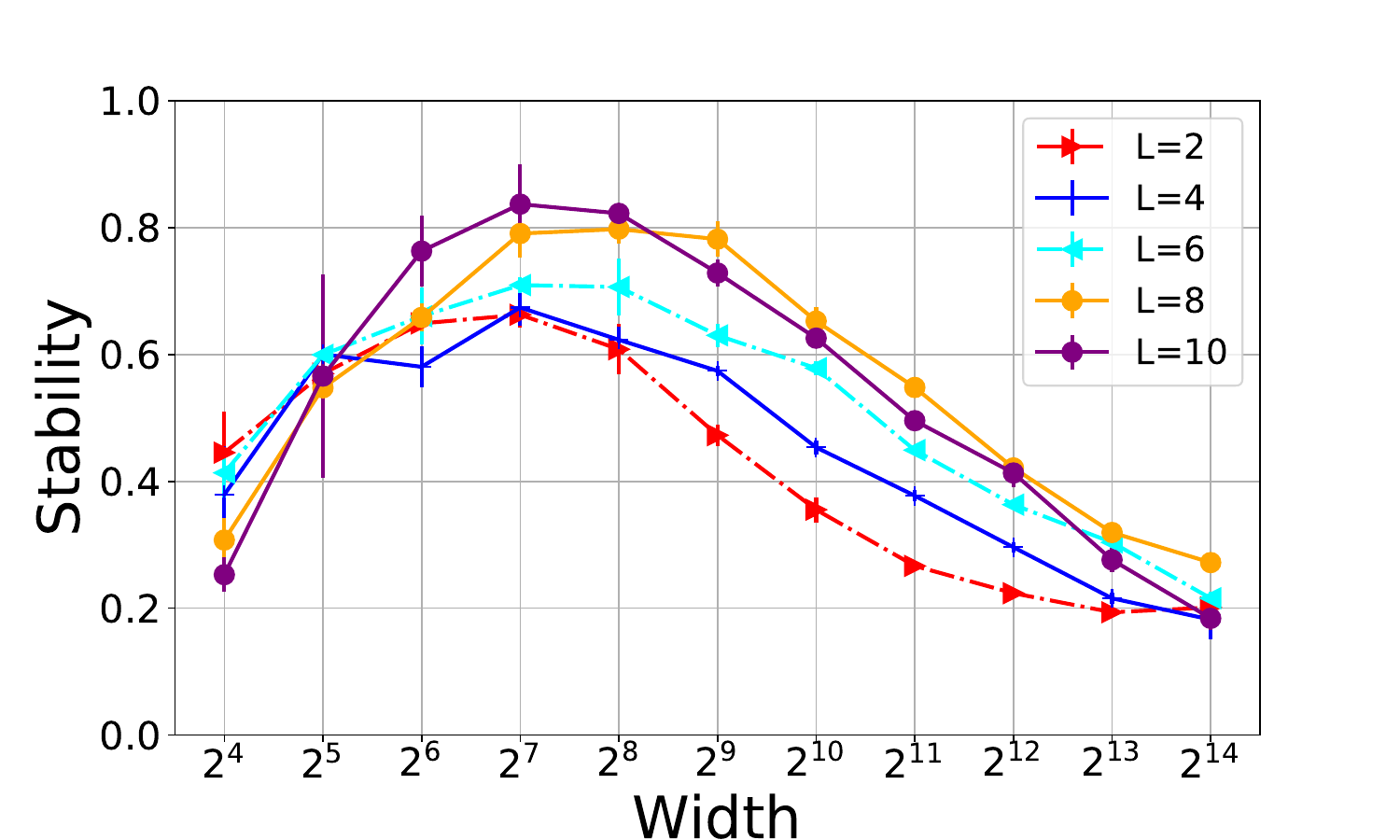}\label{fig:verify_width_lazy}}
    \subfigure[Lazy and non-lazy training under different widths]{\includegraphics[width=0.49\linewidth]{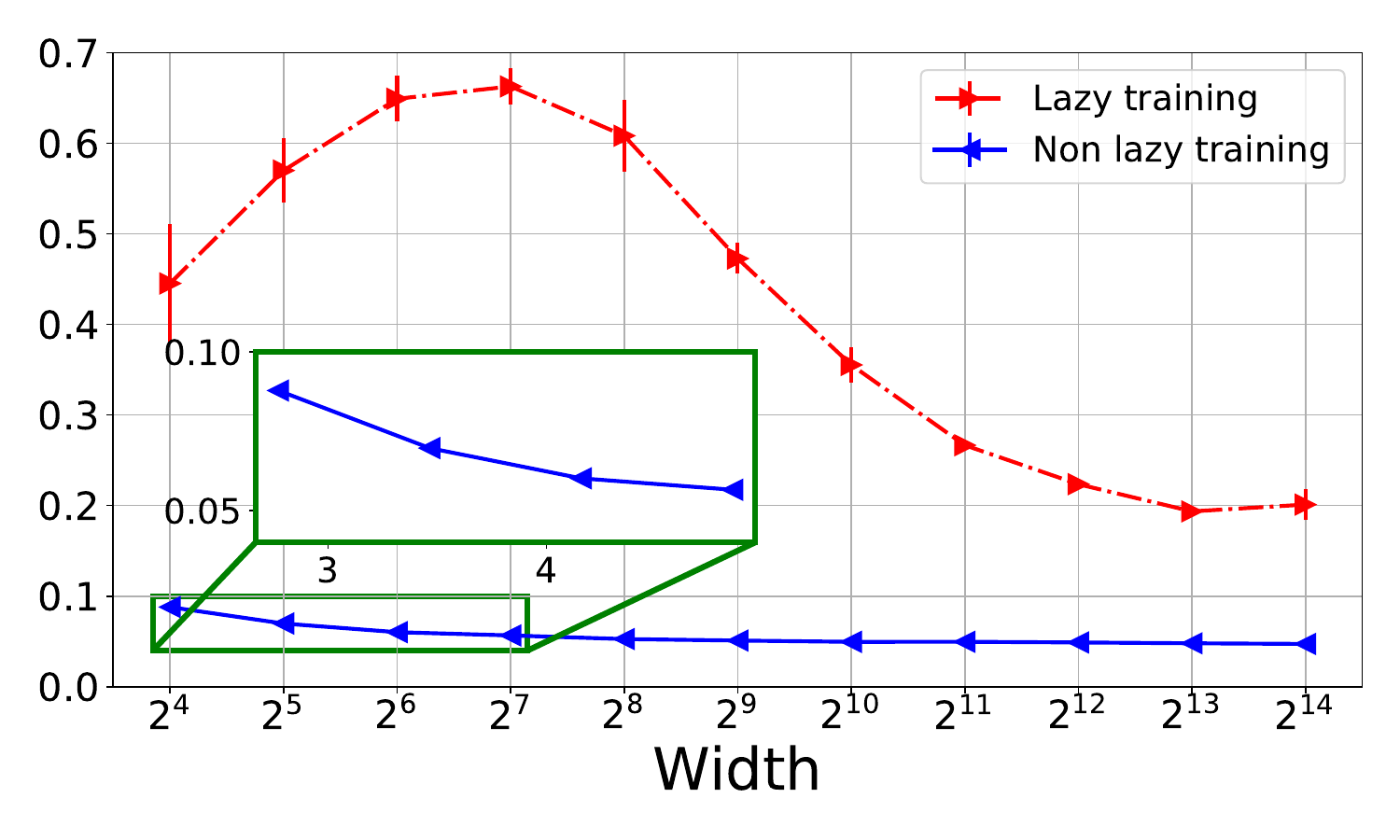}
    \label{fig:verify_width_non_lazy}}
\caption{Influence of width of neural network on the \emph{perturbation stability}. (a) phase transition of the \emph{perturbation stability} \emph{vs.} width with five different depths under lazy training. (b) the difference between lazy training and non-lazy training regimes for two layer neural networks.}\vspace{-0.25cm}
\label{fig:verify_width}
\end{figure}

\subsection{Validation for depth and initialization}
\label{ssec:Validation_depth_initialization}

Let us explore the effect of depth on the \emph{perturbation stability} under lazy training regime with different initializations in~\cref{fig:verify_init_1} and~\cref{fig:verify_init_2}. Our results show the tendency of the \emph{perturbation stability} for FCN with different widths and depths under the He initialization and the LeCun initialization, respectively. 
We observe a similar phase transition phenomenon, and find that, the \emph{perturbation stability} under He initialization increases with depth, while the LeCun initialization shows the opposite tendency, which verified our theory.

\begin{figure}[t]
\centering
    \subfigure[]{\includegraphics[width=0.5\linewidth]{figures/Stability_FCN_He.pdf}\label{fig:verify_init_1}}
    \subfigure[]{\includegraphics[width=0.47\linewidth]{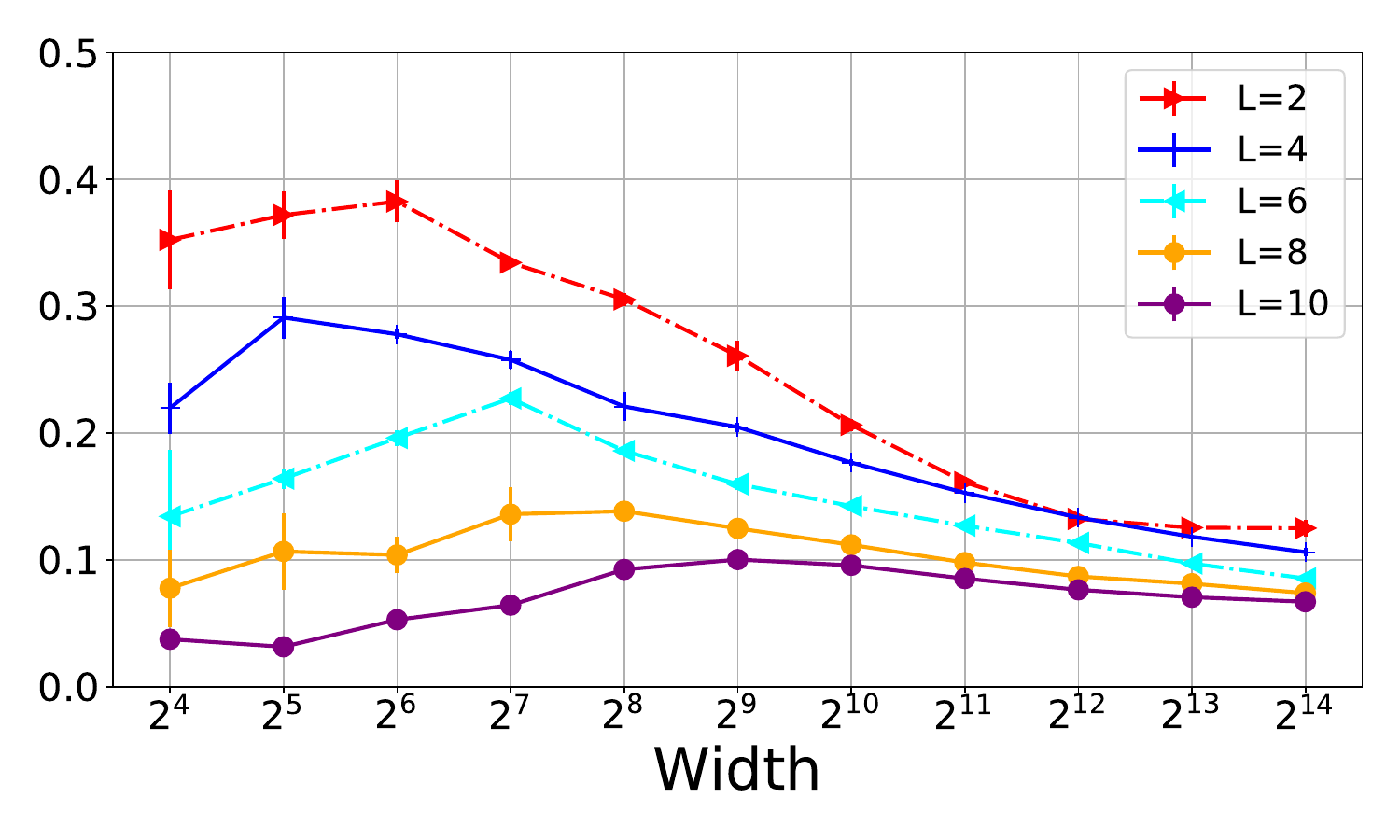}
    \label{fig:verify_init_2}}
\caption{Relationship between the \emph{perturbation stability} and depth of FCN under the He initialization (a) and the LeCun initialization (b) different depths of $L=2, 4, 6, 8$ and $10$.}
\label{fig:verify_init}
\end{figure}
\section{Conclusions}
In this work, we explore the interplay of the width, the depth and the initialization of neural networks on their average robustness with new theoretical bounds in an effort to address the apparent contradiction in the literature. Our theoretical results hold in both the under- and over-parameterized regimes. Intriguingly, we find a change of behavior in average robustness with respect to the depth, initially exacerbating robustness and then alleviating it. We suspect that this could help explain the contradictory messages in the literature. We also characterize the average robustness in the non-lazy training regime for two layer neural networks and find that width always help, coinciding with the results \citep{bubeck2021a, pmlr-v134-bubeck21a}. We also provide numerical evidence to support the theoretical developments. Our results demonstrate the rationale behind robustness of over-parameterized neural networks and might be beneficial to analyze the highly over-parameterized foundation models, which have demonstrated exceptional performance in zero-shot performance~\citep{NEURIPS2020_1457c0d6}.
\section*{Acknowledgements}
\label{sec:acks}

We are thankful to Leello Dadi and Jiayuan Ye for their insightful comments on the paper. We are also thankful to the reviewers for providing constructive feedback. Research was sponsored by the Army Research Office and was accomplished under Grant Number W911NF-19-1-0404. This work was supported by Hasler Foundation Program: Hasler Responsible AI (project number 21043). This work was supported by SNF project - Deep Optimisation of the Swiss National Science Foundation (SNSF) under grant number 200021\_205011. This work was supported by Zeiss. This project has received funding from the European Research Council (ERC) under the European Union's Horizon 2020 research and innovation programme (grant agreement n° 725594 - time-data). Corresponding authors: Fanghui Liu and Zhenyu Zhu.
\newpage
\bibliography{literature}
\bibliographystyle{abbrvnat}

\newpage
\appendix
\onecolumn
\allowdisplaybreaks
\section*{Appendix introduction} 

The Appendix is organized as follows:
\begin{itemize}
    \item In~\cref{sec:symbols_and_notations}, we state the symbols and notation used in this paper.
    \item In~\cref{sec:proofs_thm1}, we provide the proofs and related lemmas of~\cref{thm:perturbation_stability_lazy}.
    \item In~\cref{sec:proofs_thm2}, we provide the proofs of~\cref{thm:sufficient_condition_of_non-lazy_training}.
    \item In~\cref{sec:proofs_thm3}, we provide the proofs and related lemmas of~\cref{thm:perturbation_stability_non_lazy}.
    \item In~\cref{sec:additional_experiments}, we detail our experimental settings and exhibit additional experimental results.
    \item In~\cref{sec:limitation_and_discussion}, we discuss several limitations of this work.
    \item Finally, in~\cref{sec:societal_impact}, we discuss the societal impact of this paper.
\end{itemize}

\section{Symbols and Notation}
\label{sec:symbols_and_notations}
In the paper, vectors are indicated with bold small letters, matrices with bold capital letters. To facilitate the understanding of our work, we include the some core symbols and notation in \cref{table:symbols_and_notations}. 

\begin{table}[ht]

\caption{Core symbols and notations used in this project.}
\label{table:symbols_and_notations}
\small
\centering
\begin{tabular}{c | c | c}
\toprule
Symbol & Dimension(s) & Definition \\
\midrule
$\mathcal{N}(\mu,\sigma) $ & - & Gaussian distribution of mean $\mu$ and variance $\sigma$ \\
$\mathrm{Ber}(m,p) $ & - & Bernoulli (Binomial) distribution with $m$ trials and $p$ success rate. \\
$\chi^2(\omega) $ & - & Chi-square distribution of degree $\omega$. \\
\midrule
$\left \| \bm{v} \right \|_2$ & - & Euclidean norms of vectors $\bm{v}$ \\
$\left \| \bm{M} \right \|_2$ & - & Spectral norms of matrices $\bm{M}$ \\
$\left \| \bm{M} \right \|_{\mathrm{F}}$ & - & Frobenius norms of matrices $\bm{M}$ \\
$\left \| \bm{M} \right \|_{\mathrm{\ast}}$ & - & Nuclear norms of matrices $\bm{M}$ \\
$\lambda(\bm{M})$ & - & Eigenvalues of matrices $\bm{M}$ \\
$\bm{M}^{[l]}$ & - & $l$-th row of matrices $\bm{M}$\\
$\bm{M}_{i,j}$ & - & $(i, j)$-th element of matrices $\bm{M}$\\
\midrule
$\phi(x) = \max(0, x)$ & - & ReLU activation function for scalar \\
$\phi(\bm{v}) = (\phi(v_1), \dots, \phi(v_m))$ & - & ReLU activation function for vectors \\
$1\left \{A\right \}$ & - & Indicator function for event $A$\\
\midrule
$n$ & - & Size of the dataset \\
$d$ & - & Input size of the network \\
$o$ & - & Output size of the network \\
$L$ & - & Depth of the network \\
$m$ & - & Width of intermediate layer\\
$\beta_l$ & - & Standard deviation of Gaussian initialization of $l$-th intermediate layer \\
$\alpha$ & - & Scale factor for the output layer \\
\midrule
$\bm{x}_i$ & $\realnum^{d}$ & The $i$-th data point \\
$\bm{y}_i$ & $\realnum^{o}$ & The $i$-th target vector \\
$\mathcal{D}_{X}$ & - & Input data distribution \\
$\mathcal{D}_{Y}$ & - & Target data distribution \\
\midrule
$\bm{W}_1$ & $\realnum^{m \times d}$ & Weight matrix for the input layer \\
$\bm{W}_l$ & $\realnum^{m \times m}$ & Weight matrix for the $l$-th hidden layer \\
$\bm{W}_L$ & $\realnum^{o \times m}$ & Weight matrix for
the output layer \\
\midrule
$\bm{h}_{i,l}$ & $\realnum^{m}$ & The $l$-th layer activation for input $\bm{x}_i$\\
$\bm{f}_{i}$ & $\realnum^{o}$ & Output of network for input $\bm{x}_i$\\
\midrule
$\mathcal{O}$, $o$, $\Omega$ and $\Theta$ & - & Standard Bachmann–Landau order notation\\
\midrule
$\mathbb{P}(A)$ & - & Probability of event $A$\\
\midrule
\end{tabular}
\end{table}

\section{Proof of upper bound of the Perturbation Stability in lazy training regime for deep neural network} 
\label{sec:proofs_thm1}

We present the details of our results from~\cref{ssec:bound_lazy} in this section. Firstly, we introduce some lemmas in~\cref{ssec:relevant_lemmas_lazy} to facilitate the proof of theorems. Then, in~\cref{ssec:proof_lazy} we provide the proof of~\cref{thm:perturbation_stability_lazy}.

\subsection{Relevant Lemmas}
\label{ssec:relevant_lemmas_lazy}

\begin{lemma}
\label{lemma:gaussian_square_distribution}
Let $\bm{w} \sim \mathcal{N}(\bm{0},\sigma^2 \mathbb{I}_n)$. Then, for two fixed non-zero vectors $\bm{h}_1 \in \mathbb{R}^n $ and $\bm{h}_2 \in \mathbb{R}^n$, define two random variables $X = (\bm{w}^{\top}\bm{h}_1 1\left \{\bm{w}^{\top}\bm{h}_2\geq 0\right \})^2$ and $Y = s(\bm{w}^{\top}\bm{h}_1)^2$, where $s\sim \mathrm{Ber}(1,1/2)$ follows a Bernoulli distribution with $1$ trial and $\frac{1}{2}$ success rate. Then $X$ and $Y$ have the same distribution, denoted as $X \overset{d}{=} Y$.
\end{lemma}

\begin{proof}
Firstly, we derive the cumulative distribution function (CDF) of $X$. Obviously, $X$ is non-negative and $\bm w^{\!\top} \bm h_1 \sim \mathcal{N}(\bm 0, \sigma^2 \| \bm h_1\|_2^2 \mathbb{I}_n)$, and then we have:
\begin{equation*}
\begin{split}
\mathbb{P}(X = 0) & = \mathbb{P}(\bm{w}^{\top}\bm{h}_2 < 0) + \mathbb{P}(\bm{w}^{\top}\bm{h}_2 \geq 0, \bm{w}^{\top}\bm{h}_1 = 0) \,,
\end{split} 
\end{equation*}
which implies:
\begin{equation*}
\begin{split}
 \mathbb{P}(\bm{w}^{\top}\bm{h}_2 < 0) \leq \mathbb{P}(X = 0)  \leq  \mathbb{P}(\bm{w}^{\top}\bm{h}_2 < 0) + \mathbb{P}( \bm{w}^{\top}\bm{h}_1 = 0) 
& = \mathbb{P}(\bm{w}^{\top}\bm{h}_2 < 0) \,,
\end{split} 
\end{equation*}
leading to $\mathbb{P}(X = 0) = \mathbb{P}(\bm{w}^{\top}\bm{h}_2 < 0) = 1/2$.

Accordingly, for $x \geq 0$, we have: 
\begin{equation*}
\begin{split}
    \mathbb{P}(X \leq x) & = \mathbb{P}(\bm{w}^{\top}\bm{h}_2 < 0) + \mathbb{P}(\bm{w}^{\top}\bm{h}_2 \geq 0, -\sqrt{x} \leq \bm{w}^{\top}\bm{h}_1 \leq \sqrt{x})\\
    & = \frac{1}{2} + \mathbb{P}(\bm{w}^{\top}\bm{h}_2 \geq 0, -\sqrt{x} \leq \bm{w}^{\top}\bm{h}_1 \leq \sqrt{x})\\
    & = \frac{1}{2} + \mathbb{P}(\bm{w}^{\top}\bm{h}_2 \geq 0, -\sqrt{x} \leq \bm{w}^{\top}\bm{h}_1 \leq 0) + \mathbb{P}(\bm{w}^{\top}\bm{h}_2 \geq 0, 0 \leq \bm{w}^{\top}\bm{h}_1 \leq \sqrt{x})\\
    & = \frac{1}{2} + \mathbb{P}(\bm{w}^{\top}\bm{h}_2 \leq 0, 0 \leq \bm{w}^{\top}\bm{h}_1 \leq \sqrt{x}) + \mathbb{P}(\bm{w}^{\top}\bm{h}_2 \geq 0, 0 \leq \bm{w}^{\top}\bm{h}_1 \leq \sqrt{x})\\
    & = \frac{1}{2} + \mathbb{P}(0 \leq \bm{w}^{\top}\bm{h}_1 \leq \sqrt{x})\\
    & =\frac{1}{2}+\int_{0}^{\sqrt{x}} \frac{1}{\sqrt{2\pi\sigma_1^2}}e^{-\frac{t^2}{2\sigma_1^2}}\mathrm{d}t\,,
\end{split}
\end{equation*}
where we use the symmetry of the Gaussian random variable and $\sigma_1 = \left \| \bm{h}_1 \right \|_2 \sigma$.

Then $X$ admits the following cumulative distribution function:

\begin{equation}
    F(X \leq x) = \begin{cases}
0  & \text{ if } x < 0 \\
\frac{1}{2}  & \text{ if } x = 0\\
\frac{1}{2}+\int_{0}^{\sqrt{x}} \frac{1}{\sqrt{2\pi\sigma_1^2}}e^{-\frac{t^2}{2\sigma_1^2}}dt  & \text{ if } x > 0\,.
\end{cases}
\label{eq:CDF_1}
\end{equation}

We then derive the CDF of $Y$.  Obviously, $Y$ is non-negative and $\mathbb{P}(Y = 0) = 1/2$, which holds by $Y = 0$ iff $s = 0$. Accordingly, for $x \geq 0$, we have: 
\begin{equation*}
\begin{split}
\mathbb{P}(Y \leq x) & = \mathbb{P}(s = 0) + \mathbb{P}(s = 1)\mathbb{P}(-\sqrt{x} \leq \bm{w}^{\top}\bm{h}_1 \leq \sqrt{x})=\frac{1}{2}+\frac{1}{2}\int_{-\sqrt{x}}^{\sqrt{x}} \frac{1}{\sqrt{2\pi\sigma_1^2}}e^{-\frac{t^2}{2\sigma_1^2}} \mathrm{d}t \\
& = \frac{1}{2}+\int_{0}^{\sqrt{x}} \frac{1}{\sqrt{2\pi\sigma_1^2}}e^{-\frac{t^2}{2\sigma_1^2}}\mathrm{d}t \,.
\end{split}
\end{equation*}

Then $Y$ has the following cumulative distribution function:

\begin{equation}
    F(Y \leq x) = \begin{cases}
0  & \text{ if } x < 0 \\
\frac{1}{2}  & \text{ if } x = 0\\
\frac{1}{2}+\int_{0}^{\sqrt{x}} \frac{1}{\sqrt{2\pi\sigma_1^2}}e^{-\frac{t^2}{2\sigma_1^2}}dt  & \text{ if } x > 0\,,
\end{cases}
\label{eq:CDF_2}
\end{equation}
which implies $X \overset{d}{=} Y$ by comparing~\cref{eq:CDF_1} and~\cref{eq:CDF_2}.
\end{proof}

\begin{lemma}
\label{lemma:chi_square_distribution}
Given two fixed non-zero vectors $\bm{h}_1\in\realnum^p$ and $\bm{h}_2\in\realnum^p$, let $\bm{W}\in \realnum^{q\times p}$ be random matrix with i.i.d. entries $\bm{W}_{i,j} \sim \mathcal{N}(0,2/q)$ and a vector  $\bm{v} = \phi' (\bm{W}\bm{h}_2)\bm{W}\bm{h}_1 \in \realnum^q$. Then, we have $\frac{q\left \| \bm{v} \right \|_2^2 }{2\left \| \bm{h}_1 \right \|_2^2}\sim \chi^2(\varrho )$, where $\varrho  \sim \mathrm{Ber}(q,1/2)$.
\end{lemma}

\begin{proof}
According to the definition of $\bm{v} = \phi' (\bm{W}\bm{h}_2)\bm{W}\bm{h}_1 \in \realnum^q$, we have:

\begin{equation*}
    \left \| \bm{v} \right \|_2^2  = \sum_{i=1}^{q}\bigg( \bm{D}_{i,i}\left \langle \bm{W}^{[i]},\bm{h}_1 \right \rangle\bigg)^2\,,
\end{equation*}
where $\bm{D}_{i,i} = 1\left \{ \left \langle \bm{W}^{[i]},\bm{h}_2 \right \rangle \geq 0 \right \}$, the $\bm{W}^{[i]}$ is defined in the second part of the~\cref{table:symbols_and_notations}.

Let $\varpi_i = \left \langle \bm{W}^{[i]},\bm{h}_1 \right \rangle/\bigg( \sqrt{\frac{2\left \| \bm{h}_1 \right \|_2^2 }{q}}\bigg)$, then $\varpi_i \sim \mathcal{N}(0,1)$ independently. Accordingly, by~\cref{lemma:gaussian_square_distribution}, recall $s\sim \mathrm{Ber}(1,1/2)$, we have:
\begin{equation*}
\frac{q\left \| \bm{v} \right \|_2^2}{2\left \| \bm{h}_1 \right \|_2^2}  = \sum_{i=1}^{q}\bigg( 1\left \{ \left \langle \bm{W}^{[i]},\bm{h}_2 \right \rangle \geq 0 \right \}\varpi_i\bigg)^2 \overset{d}{=} \sum_{i=1}^{q} s \varpi_i^2 \,,
\end{equation*}
which implies $\frac{q\left \| \bm{v} \right \|_2^2 }{2\left \| \bm{h}_1 \right \|_2^2}\sim \chi^2(\varrho )$ with $\varrho  \sim \mathrm{Ber}(q,1/2)$ according to the definition of chi-square distribution.
\end{proof}

\begin{lemma}
\label{lemma:training_dynamics}
(Dynamic equivalence under different scaling) 
Given an $L$-layer neural network $\bm{f}$ defined by~\cref{eq:deep_network}, as follows:
\begin{equation}
\bm{f}(\bm{x}) = \widehat{\bm{W}}_L\phi(\widehat{\bm{W}}_{L-1}\cdots\phi(\widehat{\bm{W}_1}\bm{x})\cdots)\,,
\end{equation}
where $[\widehat{\bm{W}}_l]_{i,j}$ satisfy the initialization in~\cref{ssec:network}, i.e., $\beta : = \beta_2 = \dots = \beta_{L-1}$.

Scaling all weights of $\bm{f}$, then we get a new model $\widetilde{\bm{f}}$ as follows.
\begin{equation}
\widetilde{\bm{f}}(\bm{x}) = \gamma^L \widetilde{\bm{W}}_L\phi(\widetilde{\bm{W}}_{L-1}\cdots\phi(\widetilde{\bm{W}_1}\bm{x})\cdots)\,,
\end{equation}
where $[\widetilde{\bm{W}}_l]_{i,j} = [\widehat{\bm{W}}_l]_{i,j}/\gamma \quad \forall l \in [L]$.

Then if we choose an appropriate learning rate $\widetilde{\eta} := \frac{\eta}{\gamma^2} $, $\bm{f}$ and $\widetilde{\bm{f}}$ will have the same dynamics.
\end{lemma}

\begin{proof}

According to the chain rule, we have:
\begin{equation*}
    \frac{\mathrm{d} \widetilde{\bm{f}}}{\mathrm{d} \widetilde{\bm{W}}_l} = \gamma \frac{\mathrm{d} \bm{f}}{\mathrm{d} \widehat{\bm{W}}_l}  \quad \forall l \in [L]\,. 
\end{equation*}

If we choose learning rate $\widetilde{\eta} := \frac{\eta}{\gamma^2} $, then, we have:

\begin{equation*}
    \frac{\mathrm{d} \widetilde{\bm{W}}_l}{\mathrm{d} t} = \frac{1}{\gamma} \frac{\mathrm{d} \widehat{\bm{W}}_l}{\mathrm{d} t} \quad \forall l \in [L] \,.
\end{equation*}

Consider that $\widetilde{\bm{W}}_l(0) = \frac{1}{\gamma}\widehat{\bm{W}}_l(0) \quad \forall l \in [L]$, then, we have:
\begin{equation*}
    \widetilde{\bm{W}}_l(t) = \frac{1}{\gamma}\widehat{\bm{W}}_l(t) \quad \forall l \in [L]\,.
\end{equation*}

That means $\bm{f}(t) = \widetilde{\bm{f}}(t) \quad \forall t \geq 0$, which concludes the proof.
\end{proof}

\begin{lemma}
\label{lemma:lazy_training}
Given an $L$-layer neural network $\bm{f}$ defined by~\cref{eq:deep_network} trained by $\left \{ (\bm{x}_i, \bm{y}_i) \right \} _{i=1}^n$, under a small perturbation $\epsilon$, we have:
\begin{equation}
\small
\mathbb{E}_{\bm{x},\hat{\bm{x}},\bm{W}}\left \| \nabla_{\bm{x}} \bm{f}(\bm{x})^{\top}(\bm{x}-\hat{\bm{x}}) - \bm{W}_L\bm{D}_{L-1}\bm{W}_{L-1}\cdots\bm{D}_1\bm{W}_1(\bm{x}-\hat{\bm{x}})\right \|_2 \leq  \Theta\bigg(\epsilon\gamma^{L-2}\sqrt{\frac{\pi L^3m^2\beta_1^2\beta_L^2}{8}}e^{-m/L^3} \bigg)\,,
\end{equation}
where $[\bm{W}_l]_{i,j}$ satisfy the initialization in~\cref{ssec:network}, $\bm{x} \sim \mathcal{D}_{X}$ and $\hat{\bm{x}} \sim \text{Unif}(\mathbb{B} (\epsilon, \bm{x}))$.
\end{lemma}

\begin{proof}
We set the weight of the neural network after training are $\widehat{\bm{W}}$. i.e.

\begin{equation*}
    \bm{f}(\bm{x}) = \widehat{\bm{W}}_L\phi(\widehat{\bm{W}}_{L-1}\cdots\phi(\widehat{\bm{W}_1}\bm{x})\cdots)\,.
\end{equation*}

According to the standard chain rule and \cref{lemma:training_dynamics}, we have: 
\begin{equation*}
    \nabla_{\bm{x}}\bm{f}(\bm{x})^{\top}=\widehat{\bm{W}}_L\widehat{\bm{D}}_{L-1}\widehat{\bm{W}}_{L-1}\cdots\widehat{\bm{D}}_1\widehat{\bm{W}}_1 = \gamma^L \widehat{\bm{W}}'_L\widehat{\bm{D}}_{L-1}\widehat{\bm{W}}'_{L-1}\cdots\widehat{\bm{D}}_1\widehat{\bm{W}}'_1\,,
\end{equation*}

where $[\widehat{\bm{W}}'_l]_{i,j} = [\widehat{\bm{W}}_l]_{i,j}/\gamma \quad \forall l \in [L]$.

Assume that the perturbation matrices satisfy $\left \| \widehat{\bm{W}}_l - \bm{W}_l \right \|_2 \leq \omega\,, \ \forall l \in [L]$, where the parameter $\omega$ will be determined later. Then by~\citet[Lemma 7.4, Lemma 8.6, Lemma 8.7]{allen2019convergence}, we obtain that for any integer $s\in \left [ \Omega(\frac{d}{\log{m}}),\mathcal{O}(\frac{m}{L^3\log{m}}) \right ]$, for $d \leq \mathcal{O}(\frac{m}{L\log{m}})$, with probability at least $1-\exp{\bigg(-\Omega(s \log m)\bigg)}$ over the randomness of $\{ \bm W \}_{l=1}^L$, it holds that:
\begin{equation*}
\left \| \widehat{\bm{W}}'_L\widehat{\bm{D}}_{L-1}\widehat{\bm{W}}'_{L-1}\cdots\widehat{\bm{D}}_1\widehat{\bm{W}}'_1 - \bm{W}'_L\bm{D}_{L-1}\bm{W}'_{L-1}\cdots\bm{D}_1\bm{W}'_1\right \|_2 \leq  \mathcal{O}\bigg(\sqrt{\frac{L^3 s \log{m}+\omega^2 L^3 m}{d}}\sqrt{\frac{dm}{2}}\frac{\beta_1 \beta_L}{\gamma^2}\bigg)\,,
\end{equation*}
which implies that 

\begin{equation*}
\left \| \nabla_{\bm{x}} f(\bm{x})^{\top} - \bm{W}_L\bm{D}_{L-1}\bm{W}_{L-1}\cdots\bm{D}_1\bm{W}_1\right \|_2 \leq \mathcal{O}\bigg(\sqrt{\frac{L^3 s \log{m}+\omega^2 L^3 m}{d}}\sqrt{\frac{dm}{2}}\beta_1 \beta_L\gamma^{L-2}\bigg)\,,
\end{equation*}
holds with probability at least $1-\exp{\bigg(-\Omega(s \log m)\bigg)}$.

If we choose $s := \lfloor \frac{m}{L^3 \log{m}} + \frac{\omega^2}{\log{m}} \rfloor$, then, we have:

\begin{equation*}
\left \| \nabla_{\bm{x}} f(\bm{x})^{\top} - \bm{W}_L\bm{D}_{L-1}\bm{W}_{L-1}\cdots\bm{D}_1\bm{W}_1\right \|_2 \leq \mathcal{O}\bigg(\sqrt{\frac{L^3\omega^2 + m+\omega^2 L^3 m}{d}}\sqrt{\frac{dm}{2}}\beta_1 \beta_L\gamma^{L-2}\bigg)\,,
\end{equation*}
with probability at least $1-\exp{\bigg(-\Omega(\frac{m}{L^3} + \omega^2)\bigg)}$.

Let $\delta := \sqrt{\frac{L^3\omega^2 + m+\omega^2 L^3 m}{d}}\sqrt{\frac{dm}{2}}\beta_1 \beta_L\gamma^{L-2}$, we have $\omega^2 = \frac{u\delta^2-m}{L^3(m+1)}$, $u = \frac{2}{m\beta_1^2 \beta_L^2 \gamma^{2(L-2)}}$.
We have the following probability inequality:

\begin{equation*}
\small
\mathbb{P} \bigg(\left \| \nabla_{\bm{x}} f(\bm{x})^{\top} - \bm{W}_L\bm{D}_{L-1}\bm{W}_{L-1}\cdots\bm{D}_1\bm{W}_1\right \|_2 > \delta \bigg) \leq \exp{\bigg(-\frac{u\delta^2-m}{L^3(m+1)}-\frac{m}{L^3}\bigg)} = \exp{\bigg(-\frac{\delta^2 u+m^2}{L^3(m+1)}\bigg)}\,.
\end{equation*}
Then by the expectation integral equality~\cite[Lemma 1.2.1]{vershyninHighdimensionalProbabilityIntroduction2018}, the expectation is:
\begin{equation}
\small
\begin{split}
\mathbb{E}_{\bm{W}} \left \| \nabla_{\bm{x}} f(\bm{x})^{\top} - \bm{W}_L\bm{D}_{L-1}\bm{W}_{L-1}\cdots\bm{D}_1\bm{W}_1\right \|_2 & = \int_{0}^{+\infty} \mathbb{P} \bigg(\left \| \nabla_{\bm{x}} f(\bm{x})^{\top} - \bm{W}_L\bm{D}_{L-1}\bm{W}_{L-1}\cdots\bm{D}_1\bm{W}_1\right \|_2 > \delta \bigg)\mathrm{d}\delta\\
& \leq \int_{0}^{+\infty} \exp{\bigg(-\frac{\delta^2 u+m^2}{L^3(m+1)}\bigg)}\mathrm{d}\delta\\
& = \sqrt{\frac{\pi L^3(m+1)}{4u}}\exp{\bigg(-\frac{m^2}{(m+1)L^3}\bigg)}\\
& = \Theta\bigg(\gamma^{L-2}\sqrt{\frac{\pi L^3m^2\beta_1^2\beta_L^2}{8}}e^{-m/L^3} \bigg)\,.
\end{split}
\label{eq:lazy_expectation}
\end{equation}

Finally, by the definition of $\hat{\bm{x}}$, we have:
\begin{equation}
\small
\left \| \nabla_{\bm{x}} \bm{f}^{\top}(\bm{x})(\bm{x}-\hat{\bm{x}}) - \bm{W}_L\bm{D}_{L-1}\bm{W}_{L-1}\cdots\bm{D}_1\bm{W}_1(\bm{x}-\hat{\bm{x}})\right \|_2 \leq \epsilon \left \| \nabla_{\bm{x}} f(\bm{x})^{\top} - \bm{W}_L\bm{D}_{L-1}\bm{W}_{L-1}\cdots\bm{D}_1\bm{W}_1\right \|_2\,.
\label{eq:lazy_eps}
\end{equation}

By~\cref{eq:lazy_expectation} and~\cref{eq:lazy_eps} and consider expectations for $\bm x$ and $\bm x'$, we finish the proof.
\end{proof}

\begin{lemma}
\label{lemma:expection}
Given an $L$-layer neural network $\bm{f}$ defined by~\cref{eq:deep_network} trained by $\left \{ (\bm{x}_i, \bm{y}_i) \right \} _{i=1}^n$, under a small $\epsilon$, expectation over $\bm{x},\hat{\bm{x}}, \bm W$, we have:
\begin{equation}
\mathbb{E}_{\bm{x},\hat{\bm{x}},\bm{W}} \left \| \bm{W}_L\bm{D}_{L-1}\bm{W}_{L-1}\cdots\bm{D}_1\bm{W}_1(\bm{x}-\hat{\bm{x}})\right \|_2^2 \leq \frac{m o\beta_1^2\beta_L^2\gamma^{2(L-2)}}{2}\epsilon^2\,,
\end{equation}
where $[\bm{W}_l]_{i,j}$ satisfy the initialization in~\cref{ssec:network} and $\bm{x} \sim \mathcal{D}_{X}$, $\hat{\bm{x}} \sim \text{Unif}(\mathbb{B} (\epsilon, \bm{x}))$.
\end{lemma}

\begin{proof}
Define $\bm{t}_{l} = \bm{D}_{l}\bm{W}_l \cdots\bm{D}_1\bm{W}_1(\bm{x}-\hat{\bm{x}})$, then:
\begin{equation*}
    \mathbb{E}_{\bm{x},\hat{\bm{x}},\bm{W}} \left \| \bm{W}_L\bm{D}_{L-1}\bm{W}_{L-1}\cdots\bm{D}_1\bm{W}_1(\bm{x}-\hat{\bm{x}})\right \|_2^2 = \mathbb{E}_{\bm{x},\hat{\bm{x}},\bm{W}} \left \|  \bm{W}_{L}\bm{t}_{L-1}\right \|_2^2\,.
\end{equation*}

By~\cref{lemma:chi_square_distribution}, we have $\frac{\left \|\bm{t}_{l} \right \|_2^2 }{\beta^2 \left \| \bm{t}_{l-1} \right \|_2^2} \sim \chi^2(\varrho )$, where $\varrho \sim \mathrm{Ber}(m,1/2), \forall l = 2, \cdots ,L-1$.
By the law of total expectation $\mathbb{E}[\mathbb{E}[X|Y]] = \mathbb{E}[X]$, we have
\begin{equation*}
    \mathbb{E}_{\bm{W}} \frac{\left \| \bm{t}_{l} \right \|_2^2 }{\left \| \bm{t}_{l-1} \right \|_2^2} = \beta^2 \mathbb{E}_{\varrho} \chi^2(\varrho ) = \beta^2\mathbb{E} \varrho =  \frac{m\beta^2}{2} = \gamma^2 \,,\  \forall l = 2, \cdots ,L-1\,.
\end{equation*}

Similarly, we have:

\begin{equation*}
    \mathbb{E}_{\bm{W}} \frac{\left \| \bm{t}_1 \right \|_2^2 }{\left \| \hat{\bm{x}}-\bm{x} \right \|_2^2} = \frac{m\beta_1^2}{2}\,.
\end{equation*}

By the definition of chi-square distribution, we have $\frac{\left \| \bm{W}_{L}\bm{t}_{L-1} \right \|_2^2 }{\beta_L^2\left \| \bm{t}_{L-1} \right \|_2^2}\sim \chi^2(o)$, which means $\mathbb{E}_{\bm{W}} \left \|\bm{W}_{L}\bm{t}_{L-1} \right \|_2^2/\left \| \bm{t}_{L-1} \right \|_2^2 = o\beta_L^2$. 

Then, according to the independence among $\frac{\left \| \bm{W}_{L}\bm{t}_{L-1}\right \|_2^2}{\left \| \bm{t}_{L-1}\right \|_2^2}$, $\frac{\left \| \bm{t}_1\right \|_2^2}{\left \| \hat{\bm{x}}-\bm{x}\right \|_2^2}$, $\left \| \hat{\bm{x}}-\bm{x}\right \|_2^2$ and $\frac{\left \| \bm{t}_{l+1}\right \|_2^2}{\left \| \bm{t}_{l}\right \|_2^2} \quad \forall l \in [L-2]$, we have:

\begin{equation*}
\begin{split}
    \mathbb{E}_{\bm{x},\hat{\bm{x}},\bm{W}} \left \| \bm{W}_{L}\bm{t}_{L-1}\right \|_2^2 & = \mathbb{E}_{\bm{x},\hat{\bm{x}},\bm{W}} \frac{\left \| \bm{W}_{L}\bm{t}_{L-1}\right \|_2^2}{\left \| \bm{t}_{L-1}\right \|_2^2}\frac{\left \| \bm{t}_{L-1}\right \|_2^2}{\left \| \bm{t}_{L-2}\right \|_2^2}\cdots\frac{\left \| \bm{t}_1\right \|_2^2}{\left \| \hat{\bm{x}}-\bm{x}\right \|_2^2}\left \| \hat{\bm{x}}-\bm{x}\right \|_2^2\\
    & = \mathbb{E}_{\bm{W}} \frac{\left \| \bm{W}_{L}\bm{t}_{L-1}\right \|_2^2}{\left \| \bm{t}_{L-1}\right \|_2^2}\mathbb{E}_{\bm{W}}\frac{\left \| \bm{t}_{L-1}\right \|_2^2}{\left \| \bm{t}_{L-2}\right \|_2^2}\cdots\mathbb{E}_{\bm{W}}\frac{\left \| \bm{t}_1\right \|_2^2}{\left \| \hat{\bm{x}}-\bm{x}\right \|_2^2}\mathbb{E}_{\bm{x},\hat{\bm{x}}}\left \| \hat{\bm{x}}-\bm{x}\right \|_2^2\\
    & = \frac{m o\beta_1^2\beta_L^2\gamma^{2(L-2)}}{2}\mathbb{E}_{\bm{x},\hat{\bm{x}}}\left \| \hat{\bm{x}}-\bm{x}\right \|_2^2\,,
\end{split}
\end{equation*}

using the definition of $\hat{\bm{x}}$ which conclude the proof.
\end{proof}

\subsection{Proof of~\cref{thm:perturbation_stability_lazy}}
\label{ssec:proof_lazy}

\begin{proof}

According to the triangle inequality and the Jensen's inequality, we have:

\begin{equation*}
\begin{split}
    \mathscr{P}(\bm{f}, \epsilon) & = \mathbb{E}_{\bm{x},\hat{\bm{x}},\bm{W}} \left \| \nabla_{\bm{x}} \bm{f}(\bm{x})(\bm{x}-\hat{\bm{x}}) \right \|_2\\
    & \leq \mathbb{E}_{\bm{x},\hat{\bm{x}},\bm{W}}\left \| \nabla_{\bm{x}} \bm{f}(\bm{x})(\bm{x}-\hat{\bm{x}}) - \bm{W}_L\bm{D}_{L-1}\bm{W}_{L-1}\cdots\bm{D}_1\bm{W}_1(\bm{x}-\hat{\bm{x}})\right \|_2\\
    & + \mathbb{E}_{\bm{x},\hat{\bm{x}},\bm{W}} \left \| \bm{W}_L\bm{D}_{L-1}\bm{W}_{L-1}\cdots\bm{D}_1\bm{W}_1(\bm{x}-\hat{\bm{x}})\right \|_2\\
    & \leq \mathbb{E}_{\bm{x},\hat{\bm{x}},\bm{W}}\left \| \nabla_{\bm{x}} \bm{f}(\bm{x})(\bm{x}-\hat{\bm{x}}) - \bm{W}_L\bm{D}_{L-1}\bm{W}_{L-1}\cdots\bm{D}_1\bm{W}_1(\bm{x}-\hat{\bm{x}})\right \|_2\\
    & + \sqrt{\mathbb{E}_{\bm{x},\hat{\bm{x}},\bm{W}} \left \| \bm{W}_L\bm{D}_{L-1}\bm{W}_{L-1}\cdots\bm{D}_1\bm{W}_1(\bm{x}-\hat{\bm{x}})\right \|_2^2}\\
    & \lesssim \epsilon \bigg(\sqrt{L^3m^2\beta_1^2\beta_L^2}e^{-m/L^3} + \sqrt{m o \beta_1^2 \beta_L^2} \bigg)\gamma^{L-2} \,,
\end{split}
\end{equation*}

where the last inequality utilizes the results of~\cref{lemma:lazy_training} and~\cref{lemma:expection}.

\end{proof}

\section{Proof of sufficient condition for DNNs under the non-lazy training regime}
\label{sec:proofs_thm2}
In this section, we provide the proof of~\cref{thm:sufficient_condition_of_non-lazy_training}.

\begin{proof}
By~\cref{assumption:distribution_2} and by following the setting of~\citet{JMLR:v22:20-1123}, without loss of generality, we have that there exits a $T^{\star} > 0$ such that $L(\bm{W}(T^{\star})) \leq \frac{1}{32n}$ and $y_1 \geq \frac{1}{2}$. Therefore, we have:
\begin{equation*}
     \frac{1}{2n}(f_1(T^{\star}) - y_1 )^2 \leq \frac{1}{2n}\sum_{i=1}^{n}( f_i(T^{\star}) - y_i )^2 \leq L(\bm{W}(T^{\star})) \leq \frac{1}{32n}\,,
\end{equation*}

which means $\left | f_1(T^{\star}) - y_1 \right | \leq \frac{1}{4}$.
Accordingly, we conclude:
\begin{equation}
\small
\begin{split}
    \frac{1}{4} \leq y_1 - \frac{1}{4} & \leq  f_1(T^{\star})\\
    & = \frac{1}{\alpha}\bm{W}_L(T^{\star})\sigma(\bm{W}_{L-1}(T^{\star})\cdots \sigma(\bm{W}_1(T^{\star})\bm{x}_1))\\
    & = \frac{1}{\alpha}\bm{W}_L(T^{\star})\bm{D}_{1,L-1}(T^{\star})\bm{W}_{L-1}(T^{\star})\cdots \bm{D}_{1,1}(T^{\star})\bm{W}_1(T^{\star})\bm{x}_1\\
    & \leq \frac{1}{\alpha}\left \| \bm{W}_L(T^{\star}) \right \|_2 \left \| \bm{D}_{1,L-1}(T^{\star}) \right \|_2\left \| \bm{W}_{L-1}(T^{\star}) \right \|_2\cdots \left \| \bm{D}_{1,1}(T^{\star}) \right \|_2\left \| \bm{W}_1(T^{\star}) \right \|_2\left \| \bm{x}_1 \right \|_2\\
    & \leq \frac{1}{\alpha}\left \| \bm{W}_L(T^{\star}) \right \|_2 \left \| \bm{W}_{L-1}(T^{\star}) \right \|_2\cdots \left \| \bm{W}_1(T^{\star}) \right \|_2\,,
\end{split}
\label{eq:split_weight}
\end{equation}
where the last inequality uses~\cref{assumption:distribution_1} and 1-Lipschitz of ReLU.

According to~\citet[Corollary 2.1]{du2018algorithmic}, we have:
\begin{equation*}
    \frac{\mathrm{d}}{\mathrm{d} t}  (\left \| \bm{W}_1 \right \|_{\mathrm{F}}^2) = \frac{\mathrm{d}}{\mathrm{d} t}  (\left \| \bm{W}_2 \right \|_{\mathrm{F}}^2) = \cdots = \frac{\mathrm{d}}{\mathrm{d} t}  (\left \| \bm{W}_L \right \|_{\mathrm{F}}^2)\,.
\end{equation*}

Then for any $l_1, l_2 \in [L]$, we have:
\begin{equation*}
    \left \| \bm{W}_{l_1}(T^{\star}) \right \| _{\mathrm{F}}^2 - \left \| \bm{W}_{l_1}(0) \right \| _{\mathrm{F}}^2 = \left \| \bm{W}_{l_2}(T^{\star}) \right \| _{\mathrm{F}}^2 - \left \| \bm{W}_{l_2}(0) \right \| _{\mathrm{F}}^2\,,
\end{equation*}
which implies:

\begin{equation}
\begin{split}
    \left \| \bm{W}_{l_1}(T^{\star}) \right \|_2 & \leq \left \| \bm{W}_{l_1}(T^{\star}) \right \| _{\mathrm{F}}\\ 
    & = \sqrt{\left \| \bm{W}_{l_1}(T^{\star}) \right \| _{\mathrm{F}}^2}\\
    & = \sqrt{\left \| \bm{W}_{l_2}(T^{\star}) \right \| _{\mathrm{F}}^2 - \left \| \bm{W}_{l_2}(0) \right \| _{\mathrm{F}}^2 + \left \| \bm{W}_{l_1}(0) \right \| _{\mathrm{F}}^2}\\
    & \leq \sqrt{\left \| \bm{W}_{l_2}(T^{\star}) \right \| _{\mathrm{F}}^2 + \left \| \bm{W}_{l_1}(0) \right \| _{\mathrm{F}}^2}\\
    & \leq \left \| \bm{W}_{l_2}(T^{\star}) \right \| _{\mathrm{F}} + \left \| \bm{W}_{l_1}(0) \right \| _{\mathrm{F}}\,.
\end{split}
\label{eq:f_norm_2_norm}
\end{equation}

According to~\citet[Proposition 16]{JMLR:v22:20-1123} and the relationship between $\ell_2$ norm and Frobenius norm, i.e. $\left \| \cdot \right \|_{\mathrm{F}} \leq \sqrt{r} \left \| \cdot \right \|_2$, where the $r$ is the rank of matrix, with probability at least $1-(L-2)\exp(-\Theta(m^2))-\exp(-\Theta(md))-\exp(-\Theta(m))$ over the initialization, we have $\left \| \bm{W}_1(0) \right \| _{\mathrm{F}} \leq \sqrt{d}\left \| \bm{W}_1(0) \right \|_2 \leq \sqrt{\frac{3md^2}{2}}\beta_1$, $\left \| \bm{W}_l(0) \right \| _{\mathrm{F}} \leq \sqrt{m} \left \| \bm{W}_l(0) \right \|_2 \leq \sqrt{\frac{3m^3}{2}}\beta_l, \  \forall l \in [L-1]$ and $\left \| \bm{W}_L(0) \right \| _{\mathrm{F}} = \left \| \bm{W}_L(0) \right \|_2 \leq\sqrt{\frac{3m}{2}}\beta_L$.

If we combine~\cref{eq:split_weight,eq:f_norm_2_norm}, for any $l^{\star} \in [L]$, with probability at least $1-(L-2)\exp(-\Theta(m^2))-\exp(-\Theta(md))-\exp(-\Theta(m))$ over the initialization, we have:

\begin{equation*}
\begin{split}
    \frac{1}{4} & \leq \frac{1}{\alpha}\left \| \bm{W}_L(T^{\star}) \right \|_2 \left \| \bm{W}_{L-1}(T^{\star}) \right \|_2\cdots \left \| \bm{W}_1(T^{\star}) \right \|_2\\
    & = \frac{1}{\alpha}\prod_{l=1}^{L}\bigg( \left \| \bm{W}_{l^{\star}}(T^{\star}) \right \| _{\mathrm{F}} + \left \| \bm{W}_{l}(0) \right \| _{\mathrm{F}}\bigg) \\
    & \leq \frac{1}{\alpha} \bigg(  \left \| \bm{W}_{l^{\star}}(T^{\star}) \right \| _{\mathrm{F}} + \frac{1}{L}\sum_{l=1}^{L}\bigg( \left \| \bm{W}_{l}(0) \right \| _{\mathrm{F}}\bigg) \bigg)^L\\
    & \leq \frac{1}{\alpha} \bigg(  \left \| \bm{W}_{l^{\star}}(T^{\star}) \right \| _{\mathrm{F}} + \sqrt{\frac{3m^3}{2L^2}}\sum_{l=1}^{L}\beta_l \bigg)^L\,.
\end{split}
\end{equation*}

Then with probability at least $1-(L-2)\exp(-\Theta(m^2))-\exp(-\Theta(md))-\exp(-\Theta(m))$ over the initialization, we have:
\begin{equation}
\left\| \bm{W}_{l^{\star}}(T^{\star}) \right\|_{\mathrm{F}}\geq \bigg(\frac{\alpha}{4}\bigg)^{1/L}-\sqrt{\frac{3m^3}{2L^2}}\sum_{l=1}^{L}\beta_l\,.
\label{eq:bound_W_l_F}
\end{equation}

Therefore, with probability at least $1-(L-2)\exp(-\Theta(m^2))-\exp(-\Theta(md))-\exp(-\Theta(m))$ over the initialization, we have:
\begin{equation*}
\begin{split}
    \sup_{t \in [0,+\infty)} \frac{\left \| \bm{W}_l(t)-\bm{W}_l(0) \right \|_{\mathrm{F}} }{\left \| \bm{W}_l(0) \right \|_{\mathrm{F}}} & \geq \frac{\left \| \bm{W}_l(T^{\star})-\bm{W}_l(0) \right \|_{\mathrm{F}} }{\left \| \bm{W}_l(0) \right \|_{\mathrm{F}}}\\
    & \geq \frac{\left \| \bm{W}_l(T^{\star}) \right \|_{\mathrm{F}} }{\left \| \bm{W}_l(0) \right \|_{\mathrm{F}}} -1\\
    & \geq \frac{(\frac{\alpha}{4})^{1/L}-\sqrt{\frac{3m^3}{2L^2}}\sum_{i=1}^{L}\beta_i}{\sqrt{\frac{3m^3}{2}}\beta_l} -1\\
    & \geq \frac{(\frac{\alpha}{4})^{1/L}-\sqrt{\frac{3m^3}{2L^2}}\sum_{i=1}^{L}\beta_i}{\sqrt{\frac{3m^3}{2}}\sum_{i=1}^{L}\beta_i} -1\\
    & = \frac{(\frac{\alpha}{4})^{1/L}}{\sqrt{\frac{3m^3}{2}}\sum_{i=1}^{L}\beta_i} -\frac{1}{L} - 1\,,
\end{split}
\end{equation*}

where the second inequality uses triangle inequality and third inequality uses~\cref{eq:bound_W_l_F}.

If $\alpha \gg (m^{3/2}\sum_{i=1}^{L}\beta_i)^L$, then with probability at least $1-(L-2)\exp(-\Theta(m^2))-\exp(-\Theta(md))-\exp(-\Theta(m))$ over the initialization, we have:

\begin{equation*}
    \sup_{t \in [0,+\infty)} \frac{\left \| \bm{W}_l(t)-\bm{W}_l(0) \right \|_{\mathrm{F}} }{\left \| \bm{W}_l(0) \right \|_{\mathrm{F}}} \gg 1\,.
\end{equation*}
\end{proof}

\section{Proof of the perturbation stability in non-lazy training regime for two-layer networks} 
\label{sec:proofs_thm3}

Without loss of generality, we consider two-layer neural networks with a scalar output without bias. 
\begin{equation}
    f(\bm{x}) = \frac{1}{\alpha}\sum_{r=1}^{m}a_r \sigma(\bm{w}_r^{\top}\bm{x})\,,
\label{eq:2-layer_network}
\end{equation}
where $\bm{x}\in \mathbb{R}^d$, $f(\bm{x}) \in \mathbb{R}$, $\alpha$ is the scaling factor. The parameters are initialized by $a_r(0) \sim \mathcal{N}(0,\beta_2^2)$, $\bm{w}_r(0) \sim \mathcal{N}(0,\beta_1^2 \bm{I}_d)$. Our result can be extended with slight modification to the multiple-output case with bias setting. Our proof requires some additional notation, which we establish below:
\begin{equation*}
    \bm{H}_{ij}^{\infty} = \frac{m}{\alpha^2}\mathbb{E}_{\bm{w}\sim \mathcal{N}(0,\beta_1^2\bm{I}_d),\bm{a}\sim \mathcal{N}(0,\beta_2^2)} a_r^2 \bm{x}_i^{\top} \bm{x}_j 1\left \{\bm{w}_r^{\top}\bm{x}_i \geq 0, \bm{w}_r^{\top}\bm{x}_j \geq 0\right \}\,,
\end{equation*}

\begin{equation*}
    \widetilde{\bm{H}}_{i,j}(t) = \frac{1}{\alpha^2}\sum_{r=1}^{m} a_r^2(t)\mathbb{E}_{\bm{w}\sim \mathcal{N}(0,\beta_1^2\bm{I}_d)}  \bm{x}_i^{\top} \bm{x}_j 1\left \{\bm{w}_r^{\top}\bm{x}_i \geq 0, \bm{w}_r^{\top}\bm{x}_j \geq 0\right \}\,,
\end{equation*}

\begin{equation*}
    \bm{H}_{i,j}(t) = \frac{1}{\alpha^2} \sum_{r=1}^{m}a_r(t)^2 \bm{x}_i^{\top} \bm{x}_j 1\left \{\bm{w}_r(t)^{\top}\bm{x}_i \geq 0, \bm{w}_r(t)^{\top}\bm{x}_j \geq 0\right \}\,,
\end{equation*}

\begin{equation*}
    \widehat{\bm{H}}_{i,j} = \frac{1}{\alpha^2} \sum_{r=1}^{m}a_r(t)^2 \bm{x}_i^{\top} \bm{x}_j 1\left \{\bm{w}_r(0)^{\top}\bm{x}_i \geq 0, \bm{w}_r(0)^{\top}\bm{x}_j \geq 0\right \}\,,
\end{equation*}

\begin{equation*}
\bm{G}_{i,j}(t) = \frac{1}{\alpha^2} \sigma(\bm{w}_r(t)^{\top}\bm{x}_i)\sigma(\bm{w}_r(t)^{\top}\bm{x}_j)\,.
\end{equation*}

The minimum eigenvalue of $\bm{H}_{ij}^{\infty}$ is denoted as $\lambda_0$ and is assumed to be strictly greater than $0$, i.e.
\begin{equation*}
    \lambda_0 = \lambda_{\min}(\bm{H}^{\infty})> 0\,.
\end{equation*}

{\bf Remark:} This assumption follows \citet{du2018gradient} but can be proved by~\citet{pmlr-v139-nguyen21g} under the NTK initialization. Moreover,~\citet{chen2021deep, https://doi.org/10.48550/arxiv.2007.01580, bietti2019inductive} discuss this assumption in different settings.

The following two symbols are used to measure the weight changes during training:

\begin{equation}
    R_a := \frac{\alpha}{n}\sqrt{\frac{\lambda_0 }{8nm}} - \sqrt{\frac{2}{\pi}}\beta_2,\quad \text{and} \quad R_w := \frac{\alpha^2 \lambda_0  \sqrt{2\pi}\beta_1}{32n^3m(R_a(R_a+\sqrt{8/\pi}\beta_2)+\beta_2^2)}\,.
\label{eq:ra_rw}
\end{equation}

The last two symbols are used to characterize the early stages of neural network training:

\begin{equation*}
    t_1^{\star} = -\frac{2}{\lambda_0}\log\bigg(1-\frac{R_w \lambda_0 \alpha}{2\sqrt{n}(\sqrt{n}\beta_2+R_a)\left \| \bm{y}-\bm{f}(0) \right \|_2}\bigg)\,,
    \label{eq:t_1_star}
\end{equation*}

\begin{equation*}
    t_2^{\star} = -\frac{2}{\lambda_0}\log\bigg(1-\frac{R_a \lambda_0 \alpha}{2\sqrt{n}(3 \beta_1 \sqrt{\log(mn^2)} + R_w )\left \| \bm{y}-\bm{f}(0) \right \|_2}\bigg)\,.
    \label{eq:t_2_star}
\end{equation*}

Then we present the details of our results on~\cref{ssec:bound_non_lazy} in this section. Firstly, we introduce some lemmas in~\cref{ssec:relevant_lemmas_non_lazy} to facilitate the proof of theorems. Then in~\cref{ssec:proof_non_lazy} we provide the proof of~\cref{thm:perturbation_stability_non_lazy}.

\subsection{Relevant Lemmas}
\label{ssec:relevant_lemmas_non_lazy}

\begin{lemma}
\label{lemma:dynamics}
\cite[Appendix A.1]{du2018gradient} Given a two-layer neural network $f$ defined by~\cref{eq:2-layer_network} and trained by $\left \{ \bm{x}_i, y_i \right \} _{i=1}^n$ using gradient descent with the quadratic loss, let $\bm{y} = (y_1,\dots,y_n)\in \mathbb{R}^{n}$ be the label vector and $\bm{f}(t) = (f_1(t),\dots,f_n(t))\in \mathbb{R}^{n}$ be the output vector at time $t$, then, we have:
\begin{equation}
    \frac{\mathrm{d}\bm{f}(t)}{\mathrm{d}t} = (\bm{H}(t)+\bm{G}(t))(\bm{y}-\bm{f}(t))\,.
\end{equation}

\end{lemma}

\begin{proof}
Our proof here just re-organizes~\citet[Appendix A.1]{du2018gradient}.
For making our manuscript self-contained, we provide a formal proof here.

We want to minimize the quadratic loss:
\begin{equation*}
    L(\bm{W},\bm{a}) = \sum_{i=1}^{n}\frac{1}{2}[f(\bm{W},\bm{a},\bm{x}_i)-y_i]^2\,.
\end{equation*}

Using the gradient descent algorithm, the formula for update the weights is:
\begin{equation*}
\bm{W}(t+1) = \bm{W}(t) - \eta \frac{\partial L(\bm{W}(t),\bm{a}(t))}{\partial \bm{W}(t)}\,,
\end{equation*}

\begin{equation*}
\bm{a}(t+1) = \bm{a}(t) - \eta \frac{\partial L(\bm{W}(t),\bm{a}(t))}{\partial \bm{a}(t)}\,.
\end{equation*}

According to the standard chain rule, we have:
\begin{equation*}
    \frac{\partial L(\bm{W}(t),\bm{a}(t))}{\partial \bm{W}(t)} = \frac{1}{\alpha}\sum_{i=1}^{n}[f(\bm{W}(t),\bm{a}(t),\bm{x}_i)-y_i]a_r(t) 1\left \{\bm{w}_r^{\top}(t)\bm{x}_i \geq 0\right \}\bm{x}_i \,,
\end{equation*}

\begin{equation*}
\frac{\partial L(\bm{W}(t),\bm{a}(t))}{\partial \bm{a}(t)} = \frac{1}{\alpha}\sum_{i=1}^{n}[f(\bm{W}(t),\bm{a}(t),\bm{x}_i)-y_i] \sigma(\bm{w}_r^{\top}(t)\bm{x}_i)\,.
\end{equation*}

Then, we have:
\begin{equation*}
\begin{split}
    \frac{\mathrm{d} f_i(t)}{\mathrm{d} t} & = \sum_{r=1}^{m}\left \langle \frac{\partial f_i(t)}{\partial \bm{w}_r(t)} , \frac{\partial \bm{w}_r(t)}{\partial t}\right \rangle + \sum_{r=1}^{m}\frac{\mathrm{d} f_i(t)}{\mathrm{d} a_r(t)} \frac{\mathrm{d} a_r(t)}{\mathrm{d} t}\\
    & = \sum_{i=1}^{n}[y_i-f_i(t)][\bm{H}_{ij}(t)+\bm{G}_{ij}(t)]\,.
\end{split}
\end{equation*}

Written in vector form, we have:

\begin{equation*}
    \frac{\mathrm{d}\bm{f}(t)}{\mathrm{d}t} = (\bm{H}(t)+\bm{G}(t))(\bm{y}-\bm{f}(t))\,.
\end{equation*}
\end{proof}

\begin{lemma}
\label{lemma:bound_initialization}
If $\alpha \geq \frac{16n\beta_2\sqrt{\log(2n^3)}}{\lambda_0}$, with probability at least $1-\frac{1}{n}$, we have:

\begin{equation*}
    \left \| \bm{H}(0)-\bm{H}^{\infty} \right \|_2 \leq \frac{\lambda_0}{4}, \quad \text{and} \quad \lambda_{\min}(\bm{H}(0))\geq \frac{3}{4}\lambda_0\,,
\end{equation*}
\end{lemma}

{\bf Remark:} This lemma is a modified version of~\citet[Lemma 3.1]{du2018gradient}, which differs in the initialization of $\bm a$ from $ \text{Unif}(\left \{ -1, +1 \right \})$ to Gaussian initialization. 
This makes our analysis relatively intractable due to
their analysis based on $a^2_i = 1,~\forall i \in [m]$.

\begin{proof}
Firstly, for a fixed pair $(i,j)$, $\bm{H}_{ij}^{\infty}$ is an average of $\widetilde{\bm{H}}_{i,j}$ with respect to $a_r$. By Bernstein’s inequality~\citep[Chapter 2]{vershyninHighdimensionalProbabilityIntroduction2018}, with probability at least $1-\delta$, we have:
\begin{equation*}
    \left | \bm{H}_{ij}^{\infty} - \widetilde{\bm{H}}_{i,j} \right | \leq \frac{2\beta_2 \sqrt{\log(\frac{1}{\delta})}}{\alpha}\,.
\end{equation*}

Then, for fixed pair $(i,j)$, $\widetilde{\bm{H}}_{i,j}$ is an average of $\bm{H}_{ij}(0)$ with respect to $\bm{w}_r$. By Hoeffding’s inequality~\citep[Chapter 2]{vershyninHighdimensionalProbabilityIntroduction2018}, with probability at least $1-\delta'$, we have:
\begin{equation*}
    \left | \bm{H}_{ij}(0) - \widetilde{\bm{H}}_{i,j} \right | \leq \frac{2\beta_2 \sqrt{\log(\frac{1}{\delta'})}}{\alpha}\,.
\end{equation*}

Choose $\delta := \delta' := \frac{1}{2n^3}$, we have with probability at least $1-\frac{1}{n^3}$, for fixed pair $(i,j)$:
\begin{equation*}
    \left | \bm{H}_{ij}(0) - \bm{H}_{ij}^{\infty} \right | \leq \frac{4\beta_2 \sqrt{\log(2n^3)}}{\alpha}\,.
\end{equation*}

Consider the union bound over $(i,j)$ pairs, with probability at least $1-\frac{1}{n}$, we have:

\begin{equation*}
    \left | \bm{H}_{ij}(0) - \bm{H}_{ij}^{\infty} \right | \leq \frac{4\beta_2 \sqrt{\log(2n^3)}}{\alpha}\,.
\end{equation*}

Thus, we have:
\begin{equation*}
    \left \| \bm{H}(0) - \bm{H}^{\infty} \right \|_2^2 \leq \left \| \bm{H}(0) - \bm{H}^{\infty} \right \|_{\mathrm{F}}^2 \leq \sum_{i,j}\left | \bm{H}_{ij}(0) - \bm{H}_{ij}^{\infty} \right |^2 \leq \frac{16n^2\beta_2^2 \log(2n^3)}{\alpha^2}\,.
\end{equation*}

when $\alpha \geq \frac{16n\beta_2\sqrt{\log(2n^3)}}{\lambda_0}$, we have the desired result.
\end{proof}

\begin{lemma}
\label{lemma:bound_changes_of_gram_matrix}
With probability at least $1-\frac{2}{n}$ over initialization, if a set of weight vectors $\left \{ \bm{w}_r \right \}_{r=1}^m$ and the output weight $\left \{ a_r \right \}_{r=1}^m$ satisfy for all $r \in [m]$, $\left \| \bm{w}_r(t) - \bm{w}_r(0) \right \|_2 \leq R_{w}$ and $\left | a_r(t) - a_r(0) \right | \leq R_{a}$, then, we have:

\begin{equation*}
    \left \| \bm{H}(t) - \bm{H}(0) \right \|_2 \leq \frac{\lambda_0}{4}, \quad \text{and} \quad \lambda_{\min}(\bm{H}(t))\geq \frac{\lambda_0}{2}\,.
\end{equation*}
\end{lemma}

\begin{proof}

Firstly, we can derive that:

\begin{equation*}
    \widehat{\bm{H}}_{i,j}(t)-\bm{H}_{i,j}(0) = \frac{1}{\alpha^2} \sum_{r=1}^{m}(a_r(t)^2-a_r(0)^2) \bm{x}_i^{\top} \bm{x}_j 1\left \{\bm{w}_r(0)^{\top}\bm{x}_i \geq 0, \bm{w}_r(0)^{\top}\bm{x}_j \geq 0\right \}\,,
\end{equation*}

\begin{equation*}
\begin{split}
    \bm{H}_{i,j}(t)-\widehat{\bm{H}}_{i,j}(t) & = \frac{1}{\alpha^2} \sum_{r=1}^{m}a_r(t)^2 \bm{x}_i^{\top} \bm{x}_j 1\left \{\bm{w}_r(t)^{\top}\bm{x}_i \geq 0, \bm{w}_r(t)^{\top}\bm{x}_j \geq 0\right \}\\
    & - \frac{1}{\alpha^2} \sum_{r=1}^{m}a_r(t)^2 \bm{x}_i^{\top} \bm{x}_j 1\left \{\bm{w}_r(0)^{\top}\bm{x}_i \geq 0, \bm{w}_r(0)^{\top}\bm{x}_j \geq 0\right \}\,.
\end{split}
\end{equation*}

Then we can compute the expectation of $\left |\widehat{\bm{H}}_{i,j}(t)-\bm{H}_{i,j}(0)\right | $:

\begin{equation}
\begin{split}
    \mathbb{E} \left | \widehat{\bm{H}}_{i,j}(t)-\bm{H}_{i,j}(0) \right | & = \mathbb{E}  \left | \frac{1}{\alpha^2} \sum_{r=1}^{m}(a_r(t)^2-a_r(0)^2) \bm{x}_i^{\top} \bm{x}_j 1\left \{\bm{w}_r(0)^{\top}\bm{x}_i \geq 0, \bm{w}_r(0)^{\top}\bm{x}_j \geq 0\right \} \right |\\
    & \leq \frac{m}{\alpha^2} \mathbb{E} \left | a_r(t)^2-a_r(0)^2 \right |\\
    & = \frac{m}{\alpha^2} \mathbb{E} \left | (a_r(t)-a_r(0))(a_r(t)+a_r(0)) \right |\\
    & \leq \frac{m R_{a}}{\alpha^2} \mathbb{E} \left | a_r(t)+a_r(0) \right |\\
    & \leq \frac{m R_{a}}{\alpha^2} (R_{a} + 2\mathbb{E} \left | a_r(0) \right |)\\
    & \leq \frac{m (R_{a} + \mathbb{E} \left | a_r(0) \right |)^2}{\alpha^2} \\
    & = \frac{m (R_{a} + \sqrt{\frac{2}{\pi}}\beta_2)^2}{\alpha^2}\,.
\end{split}
\label{eq:bound_ht_hzero}
\end{equation}

Then we define the event:

\begin{equation*}
    A_{i,r} = \left \{ \exists : \left \| \bm{w}_r(t)- \bm{w}_r(0) \right \|\leq R_w , 1\left \{\bm{w}_r(0)^{\top}\bm{x}_i \geq 0\right \}\ne 1\left \{\bm{w}_r(t)^{\top}\bm{x}_i \geq 0\right \}\right \}\,.
\end{equation*}

This event happens if and only if $\left | \bm{w}_r(0)^{\top}\bm{x}_i \right | < R_t $. According to this, we can get $\mathbb{P}(A_{i,r}) = \mathbb{P}_{z\sim\mathcal{N}(0,\beta_1^2)}(\left | z \right |\leq R_w )\leq \frac{2R_w}{\sqrt{2\pi}\beta_1}$, further:

\begin{equation}
\begin{split}
\mathbb{E} \left | \bm{H}_{i,j}(t)-\widehat{\bm{H}}_{i,j}(t) \right | & = \frac{1}{\alpha^2}\mathbb{E}  \bigg|  \sum_{r=1}^{m}a_r(t)^2 \bm{x}_i^{\top} \bm{x}_j 1\left \{\bm{w}_r(t)^{\top}\bm{x}_i \geq 0, \bm{w}_r(t)^{\top}\bm{x}_j \geq 0\right \}\\
    & - \sum_{r=1}^{m}a_r(t)^2 \bm{x}_i^{\top} \bm{x}_j 1\left \{\bm{w}_r(0)^{\top}\bm{x}_i \geq 0, \bm{w}_r(0)^{\top}\bm{x}_j \geq 0\right \}  \bigg|\\
    & \leq \frac{1}{\alpha^2} \sum_{r=1}^{m} \mathbb{E} \bigg(a_r(t)^2 \bm{x}_i^{\top} \bm{x}_j 1\left \{ A_{i,r} \cup  A_{j,r} \right \}\bigg)\\
    & \leq \frac{1}{\alpha^2} \sum_{r=1}^{m} \mathbb{E} \bigg(a_r(t)^2 \frac{4R_w}{\sqrt{2\pi}\beta_1}\bigg)\\
    & = \frac{4R_w}{\alpha^2\sqrt{2\pi}\beta_1}\sum_{r=1}^{m} \mathbb{E} (a_r(t)^2 - a_r(0)^2 + a_r(0)^2)\\
    & \leq \frac{4R_w m}{\alpha^2\sqrt{2\pi}\beta_1}\bigg(R_a(R_a+\sqrt{\frac{8}{\pi}}\beta_2)+\beta_2^2\bigg)\,,
\end{split}
\label{eq:bound_ht_htt}
\end{equation}

where the last inequality uses the result of~\cref{eq:bound_ht_hzero}.

From~\cref{eq:bound_ht_hzero,eq:bound_ht_htt}, using Markov's inequality. with probability at least $1-\frac{2}{n}$, we have:

\begin{equation*}
     \left | \widehat{\bm{H}}_{i,j}(t)-\bm{H}_{i,j}(0) \right | \leq \frac{nm (R_{a} + \sqrt{\frac{2}{\pi}}\beta_2)^2}{\alpha^2}\,,
\end{equation*}

\begin{equation*}
     \left | \bm{H}_{i,j}(t)-\widehat{\bm{H}}_{i,j}(t) \right | \leq \frac{4R_w nm}{\alpha^2\sqrt{2\pi}\beta_1 }\bigg(R_a(R_a+\sqrt{\frac{8}{\pi}}\beta_2)+\beta_2^2\bigg)\,.
\end{equation*}

Then, we have:
\begin{equation*}
\begin{split}
    \left \| \bm{H}(t)-\bm{H}(0) \right \|_2 & \leq \left \| \bm{H}(t)-\bm{H}(0) \right \|_{\mathrm{F}}\\
    & \leq \sum_{(i,j)=(1,1)}^{(n,n)}\left | \bm{H}_{i,j}(t) - \bm{H}_{i,j}(0) \right | \\
    & \leq \sum_{(i,j)=(1,1)}^{(n,n)}\bigg(\left | \widehat{\bm{H}}_{i,j}(t)-\bm{H}_{i,j}(0) \right | + \left | \bm{H}_{i,j}(t)-\widehat{\bm{H}}_{i,j}(t) \right |\bigg)\\
    & \leq \frac{mn^3}{\alpha^2}\bigg( (R_{a} + \sqrt{\frac{2}{\pi}}\beta_2)^2 + \frac{4R_w}{\sqrt{2\pi}\beta_1}(R_a(R_a+\sqrt{\frac{8}{\pi}}\beta_2)+\beta_2^2)\bigg)\,.
\end{split}
\end{equation*}

Then, by~\cref{eq:ra_rw}, we have:
\begin{equation*}
    \left \| \bm{H}(t)-\bm{H}(0) \right \|_2 \leq \frac{\lambda_0}{4}\,,
\end{equation*}

which implies:
\begin{equation*}
    \lambda_{\min}(\bm{H}(t)) \leq \lambda_{\min}(\bm{H}(0))- \frac{\lambda_0}{4} \leq \frac{\lambda_0}{2}\,.
\end{equation*}
\end{proof}

\begin{lemma}
\label{lemma:bound_w_use_a}
    Suppose that for $0 \leq s \leq t$, $\lambda_{\min}(\bm{H}(s))\geq \frac{\lambda_0}{2}$ and $\left | a_r(s) - a_r(0) \right | \leq R_{a}$. Then with probability at least $1-n\exp(-n/2)$ over initialization, we have $\left \| \bm{w}_r(t) - \bm{w}_r(0) \right \|_2 \leq R_{w}$ for all $r \in [m]$ and the $t \leq t_1^{\star}$.
\end{lemma}

\begin{proof}
By~\cref{lemma:dynamics}, we have $\frac{\mathrm{d}\bm{f}(t)}{\mathrm{d}t} = (\bm{H}(t)+\bm{G}(t))(\bm{y}-\bm{f}(t))$. Then we can calculate the dynamics of risk function:

\begin{equation*}
\begin{split}
    \frac{\mathrm{d}}{\mathrm{d} t} \left \| \bm{y}-\bm{f}(t) \right \|_2^2 & = -2(\bm{y}-\bm{f}(t))^{\top}(\bm{H}(t)+\bm{G}(t))(\bm{y}-\bm{f}(t))\\
    & \leq -2(\bm{y}-\bm{f}(t))^{\top}(\bm{H}(t))(\bm{y}-\bm{f}(t))\\
    & \leq -\lambda_0 \left \| \bm{y}-\bm{f}(t) \right \|_2^2\,,
\end{split}
\end{equation*}
in the first inequality we use that the $\bm{G}(t)$ is Gram matrix thus it is positive. Then, we have $ \frac{\mathrm{d}}{\mathrm{d} t} \bigg(e^{\lambda_0 t}\left \| \bm{y}-\bm{f}(t) \right \|_2^2\bigg) \leq 0$, then $e^{\lambda_0 t}\left \| \bm{y}-\bm{f}(t) \right \|_2^2$ is a decreasing function with respect to $t$. Thus, we can bound the risk:

\begin{equation}
    \left \| \bm{y}-\bm{f}(t) \right \|_2^2 \leq e^{-\lambda_0 t}\left \| \bm{y}-\bm{f}(0) \right \|_2^2\,.
\end{equation}

Then we bound the gradient of $\bm{w}_r$. For $0 \leq s \leq t$, With probability at least $1-n\exp(-n/2)$, we have:

\begin{equation*}
    \begin{split}
        \left \| \frac{\mathrm{d}}{\mathrm{d} s}\bm{w}_r(s)  \right \|_2 & = \left \| \frac{1}{\alpha}\sum_{i=1}^{n}[f(\bm{W}(s),\bm{a}(s),\bm{x}_i)-y_i]a_r(s) 1\left \{\bm{w}_r^{\top}(s)\bm{x}_i \geq 0\right \}\bm{x}_i  \right \|_2\\
        & \leq \frac{1}{\alpha}\sum_{i=1}^{n} \left | f(\bm{W}(s),\bm{a}(s),\bm{x}_i)-y_i \right | \left | a_r(0) + R_a \right | \\
        & \leq \frac{\sqrt{n}}{\alpha}\left \| \bm{y}-\bm{f}(s) \right \|_2 (\sqrt{n}\beta_2 + R_a)\\
        & \leq \frac{\sqrt{n}}{\alpha}(\sqrt{n}\beta_2 + R_a)e^{-\lambda_0 s/2}\left \| \bm{y}-\bm{f}(0) \right \|_2\,,
    \end{split}
\end{equation*}

where the second inequality is because of $a_r(0) \sim \mathcal{N}(0,\beta_2^2)$, then with probability at least $1-\exp(-n/2)$, we have $a_r(0) \leq \sqrt{n} \beta_2$. Then, we have:

\begin{equation}
    \left \| \bm{w}_r(t) - \bm{w}_r(0) \right \|_2 \leq \int_{0}^{t}\left \| \frac{\mathrm{d}}{\mathrm{d} s}\bm{w}_r(s)  \right \|_2 \mathrm{d}s \leq \frac{2\sqrt{n}}{\lambda_0\alpha}(\sqrt{n}\beta_2 + R_a)\left \| \bm{y}-\bm{f}(0) \right \|_2 (1-\exp(-\frac{\lambda_0 t}{2}))\,.
    \label{eq:wt}
\end{equation}

If we account for $t$, then we conclude the proof. 
\end{proof}

\begin{lemma}
\label{lemma:bound_a_use_w}
Suppose that for $0 \leq s \leq t$, $\lambda_{\min}(\bm{H}(s))\geq \frac{\lambda_0}{2}$ and $\left \| \bm{w}_r(s) - \bm{w}_r(0) \right \|_2 \leq R_{w}$. Then with probability at least $1-\frac{1}{n}$ over initialization, we have $\left | a_r(t) - a_r(0) \right | \leq R_{a}$ for all $r \in [m]$ and the $t \leq t_2^{\star}$.
\end{lemma}

\begin{proof}

Note for any $i \in [n]$ and $r \in [m]$, $\bm{w}_r^{\top}(0)\bm{x}_i ~\sim \mathcal{N}(0, \beta_1^2)$. Therefore applying Gaussian tail bound and union bound, we have with probability at least $1 - \frac{1}{n}$, for all $i \in [n]$ and $r \in [m]$, $\left | \bm{w}_r^{\top}(0)\bm{x}_i \right |\leq 3 \beta_1 \sqrt{\log(mn^2)}$, That means for $0 \leq s \leq t$, With probability at least $1-\frac{1}{n}$, we have:

\begin{equation*}
    \begin{split}
        \left | \frac{\mathrm{d}}{\mathrm{d} s}a_r(s)  \right | & = \left | \frac{1}{\alpha}\sum_{i=1}^{n}[f(\bm{W}(t),\bm{a}(t),\bm{x}_i)-y_i] \sigma(\bm{w}_r^{\top}(t)\bm{x}_i)  \right |\\
        & \leq \frac{\sqrt{n}}{\alpha}\left \| \bm{y}-\bm{f}(s) \right \|_2 (\left | \bm{w}_r^{\top}(0)\bm{x}_i \right | + R_w )\\
        & \leq \frac{\sqrt{n}}{\alpha}e^{-\lambda_0 s/2}\left \| \bm{y}-\bm{f}(0) \right \|_2 \bigg(3 \beta_1 \sqrt{\log(mn^2)} + R_w \bigg)\,.
    \end{split}
\end{equation*}

Then, we have:
\begin{equation}
\small
    \left | a_r(t) - a_r(0) \right |_2 \leq \int_{0}^{t}\left | \frac{\mathrm{d}}{\mathrm{d} s}a_r(s)  \right | \mathrm{d}s \leq \frac{2\sqrt{n}}{\lambda_0\alpha}\bigg(3 \beta_1 \sqrt{\log(mn^2)} + R_w \bigg)\left \| \bm{y}-\bm{f}(0) \right \|_2 (1-\exp(-\frac{\lambda_0 t}{2}))\,.
    \label{eq:at}
\end{equation}

Bring in $t$, then finish the proof.

\end{proof}

\begin{lemma}
\label{lemma:bound_wt_at}
Suppose $0 \leq t \leq \min(t_1^{\star}, t_2^{\star})$. Then with probability at least $1-n\exp(-n/2)-\frac{3}{n}$ over initialization, we have: $\lambda_{\min}(\bm{H}(t))\geq \frac{\lambda_0}{2}$, 
\begin{equation*}
    \left | a_r(t) - a_r(0) \right | \leq \frac{2\sqrt{n}}{\lambda_0\alpha}\bigg(3 \beta_1 \sqrt{\log(mn^2)} + R_w \bigg)\left \| \bm{y}-\bm{f}(0) \right \|_2 (1-\exp(-\frac{\lambda_0 t}{2})) := R_a^{\star}(t)\,,
\end{equation*}

\begin{equation*}
    \left \| \bm{w}_r(t) - \bm{w}_r(0) \right \|_2 \leq \frac{2\sqrt{n}}{\lambda_0\alpha}(\sqrt{n}\beta_2 + R_a)\left \| \bm{y}-\bm{f}(0) \right \|_2 (1-\exp(-\frac{\lambda_0 t}{2})) := R_w^{\star}(t)\,,
\end{equation*}

for all $r \in [m]$.
\end{lemma}

\begin{proof}
When $t = 0$, $\lambda_{\min}(\bm{H}(s))\geq \frac{3}{4}\lambda_0$, $\left | a_r(t) - a_r(0) \right | = 0 < R_{a}$ and $\left \| \bm{w}_r(t) - \bm{w}_r(0) \right \|_2 = 0 < R_{w}$. Using induction, combine~\cref{lemma:bound_changes_of_gram_matrix}, ~\cref{lemma:bound_w_use_a} and ~\cref{lemma:bound_a_use_w}, we have the result.
\end{proof}

\subsection{Proof of~\cref{thm:perturbation_stability_non_lazy}}
\label{ssec:proof_non_lazy}
\begin{proof}
We can compute the gradient of the network that:

\begin{equation*}
    \nabla_{\bm{x}} f_t(\bm{x})= \frac{1}{\alpha} \sum_{r=1}^{m} a_r(t) 1\left \{\bm{w}_r^{\top}(t)\bm{x} \geq 0\right \}\bm{w}_r^{\top}(t) \,.
\end{equation*}
\end{proof}

Then we can derive that:
\begin{equation}
\begin{split}
    \mathscr{P}(f_t, \epsilon) & =  \mathbb{E}_{\bm{x},\hat{\bm{x}}} \left | \frac{1}{\alpha} \sum_{r=1}^{m} a_r(t) 1\left \{\bm{w}_r^{\top}(t)\bm{x} \geq 0\right \}\bm{w}_r^{\top}(t) (\bm{x}-\bm{\hat{x}}) \right | \\
    & \leq \frac{1}{\alpha} \mathbb{E}_{\bm{x},\hat{\bm{x}}} \sum_{r=1}^{m}\left |   a_r(t) \bm{w}_r^{\top}(t) (\bm{x}-\bm{\hat{x}})\right |\\
    & \leq \frac{1}{\alpha} \mathbb{E}_{\bm{x},\hat{\bm{x}}} \sum_{r=1}^{m}\left |   a_r(t) \right | \left \| \bm{w}_r(t) \right \|_2  \left \| \bm{x}-\bm{\hat{x}}\right \|_2\\
    & \leq \frac{\epsilon}{\alpha} \sum_{r=1}^{m}\left | a_r(t) \right | \left \| \bm{w}_r(t) \right \|_2 \,.
\end{split}
\label{eq:p_non_lazy}
\end{equation}

Then by~\cref{lemma:bound_wt_at}, we have:

\begin{equation*}
    \left |   a_r(t) \right | \leq \left | a_r(t) - a_r(0) \right | + \left | a_r(0) \right | \leq R_a^{\star}(t) + \left | a_r(0) \right |\,.
\end{equation*}

\begin{equation*}
    \left \| \bm{w}_r(t) \right \|_2 \leq \left \| \bm{w}_r(t) - \bm{w}_r(0) \right \|_2 + \left \| \bm{w}_r(0) \right \|_2 \leq R_w^{\star}(t) + \left \| \bm{w}_r(0) \right \|_2\,.
\end{equation*}

From~\cref{eq:bound_ht_hzero}, we have $ \mathbb{E} \left | a_r(0) \right | = \sqrt{\frac{2}{\pi}}\beta_2$. That means with probability at least $1-\frac{1}{n}$ over initialization, we have $\left | a_r(0) \right | \leq \sqrt{\frac{2}{\pi}}n\beta_2$.

By~\citet[Chapter 3]{vershyninHighdimensionalProbabilityIntroduction2018}, with probability at least $1-\delta$ over initialization, we have $\left \| \bm{w}_r(0) \right \|_2 \leq 4\beta_1\sqrt{m} + 2\beta_1\sqrt{\log n}$.

By combining the results above with~\cref{eq:p_non_lazy} and~\cref{lemma:bound_wt_at}, with probability at least $1-n\exp(-n/2)-\frac{3}{n}$ over initialization we obtain that:

\begin{equation}
\begin{split}
    \mathscr{P}(f_t, \epsilon) & \leq \frac{\epsilon}{\alpha} \sum_{r=1}^{m}\left |  a_r(t) \right | \left \| \bm{w}_r(t) \right \|_2 \\
    & \leq \frac{\epsilon m}{\alpha} (R_a^{\star}(t) + \sqrt{\frac{2}{\pi}}n\beta_2) (R_w^{\star}(t) + 4\beta_1\sqrt{m} + 2\beta_1\sqrt{\log n})
\end{split}
\label{eq:non_lazy_bound}
\end{equation}

Suppose that $\alpha \sim 1$, $\beta_1 \sim \beta_2 \sim \beta \sim \frac{1}{m^c}$, $c \geq 1.5$, $m\gg n^2$. Then $R_a = \Theta(\frac{1}{\sqrt{n^3m}})$, $R_w = \Theta(\frac{1}{m^c})$, $R_a^{\star}(t) = \Theta(\frac{\sqrt{n \log m}}{m^c})$ and $R_w^{\star}(t) = \Theta(\frac{1}{\sqrt{n^3m}})$. Bring these results into~\cref{eq:non_lazy_bound}, with probability at least $1-n\exp(-\frac{n}{2})-\frac{3}{n}$ over initialization, we have:

\begin{equation*}
    \mathscr{P}(f_t, \epsilon) \leq \Theta \bigg(\epsilon \frac{\sqrt{n \log m} + n}{m^{c-1}} \bigg(\frac{1}{\sqrt{n^3m}}+\frac{1}{m^{c-0.5}} \bigg)\bigg)\,.
\end{equation*}

\section{Additional Experiments}
\label{sec:additional_experiments}

A number of additional experiments are conducted in this section. Unless explicitly mentioned otherwise, the experimental setup remains similar to the one in the main paper. The following experiments are conducted below: 

\begin{enumerate}
    \item In~\cref{ssec:additional_experiments_compare_lazy_nonlazy}, we compare the two different training regimes, lazy training and non-lazy training.
    \item In~\cref{ssec:ablation_early_stop}, we conduct some experiments to assess early-stopping training.
    \item In~\cref{ssec:additional_experiments_cnn}, we extend the experiments in~\cref{ssec:Validation_width} from fully connected network to Convolutional Neural Network.
    \item In~\cref{ssec:additional_experiments_non_lazy}, we conduct experiments with additional initializations under the non-lazy training regime.
    \item In~\cref{ssec:additional_experiments_NTK}, we extend the experiments from He and LeCun initialization to NTK initialization.
    \item In~\cref{ssec:additional_experiments_ResNet}, we extend in~\cref{ssec:Validation_width} from fully connected network to ResNet.
\end{enumerate}

\subsection{Comparison of Lazy training and Non-lazy training}
\label{ssec:additional_experiments_compare_lazy_nonlazy}

In this section, we test the performance of the lazy training regime and the non-lazy training regime on the standard ResNet-110\footnote{We use the following link for the implementation: \url{https://github.com/bearpaw/pytorch-classification}.}. We adopt a narrow model width for computational efficiency. We choose the He initialization and the non-lazy training initialization as mentioned in~\cref{ssec:experiment_setting} on CIFAR10 and CIFAR100. The results are provided in~\cref{table:resnet_lazy_non_lazy}. Notice that the non-lazy training regime achieves a similar performance to the lazy training regime. This implies that the non-lazy training regime is also needed for studying practical learning tasks.

\begin{table}[t]
\caption{Compare the clean accuracy for lazy training regime and non-lazy training regime of ResNet-110.}
\label{table:resnet_lazy_non_lazy}
\centering
\begin{tabular}{c | c | c}
\hline
Dataset & Lazy training & Non-lazy training\\
\hline
CIFAR10  & $92.89\%$ & $92.14\%$ \\
CIFAR100 & $71.08\%$ & $70.55\%$ \\
\hline
\end{tabular}
\end{table}

\subsection{Ablation study on early stopping (training 50 epochs)}
\label{ssec:ablation_early_stop}

In this section, we conduct an experiment to assess the early-stopping technique that is frequently employed in neural network training. In our case, we consider stopping after 50 epochs. The experimental results shown in~\cref{fig:ablation_early_stop} indicate that the loss and accuracy of the neural network remain almost unchanged from the 50th epoch to the 200th epoch under two different network settings: width = 32, depth = 4 and width = 64, depth = 8. Therefore, we train the rest of the networks for 50 epochs in this work.

\begin{figure}[t]
\centering
    \subfigure[]{\includegraphics[width=0.49\linewidth]{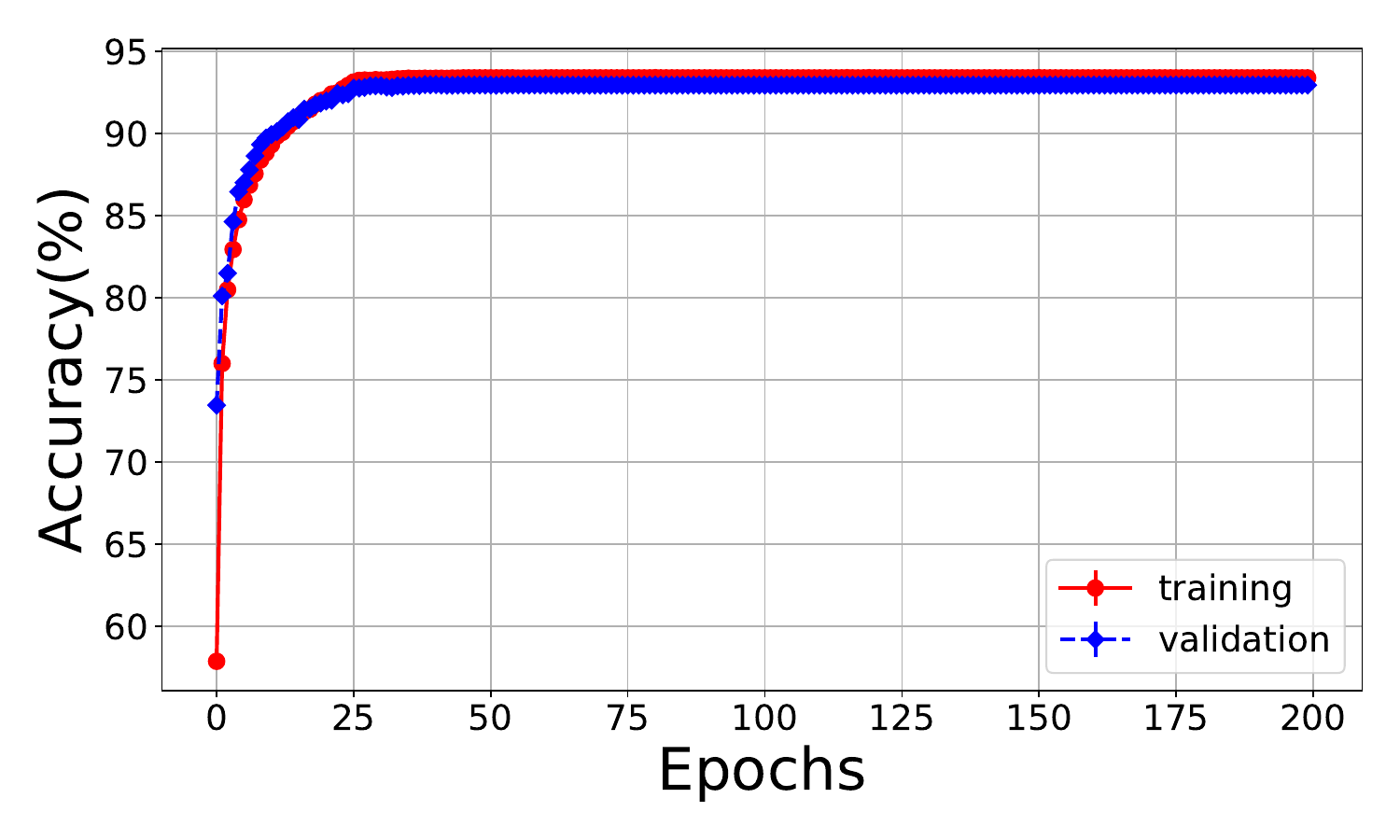}}\label{fig:ablation_early_stop_acc_4}
    \subfigure[]{\includegraphics[width=0.49\linewidth]{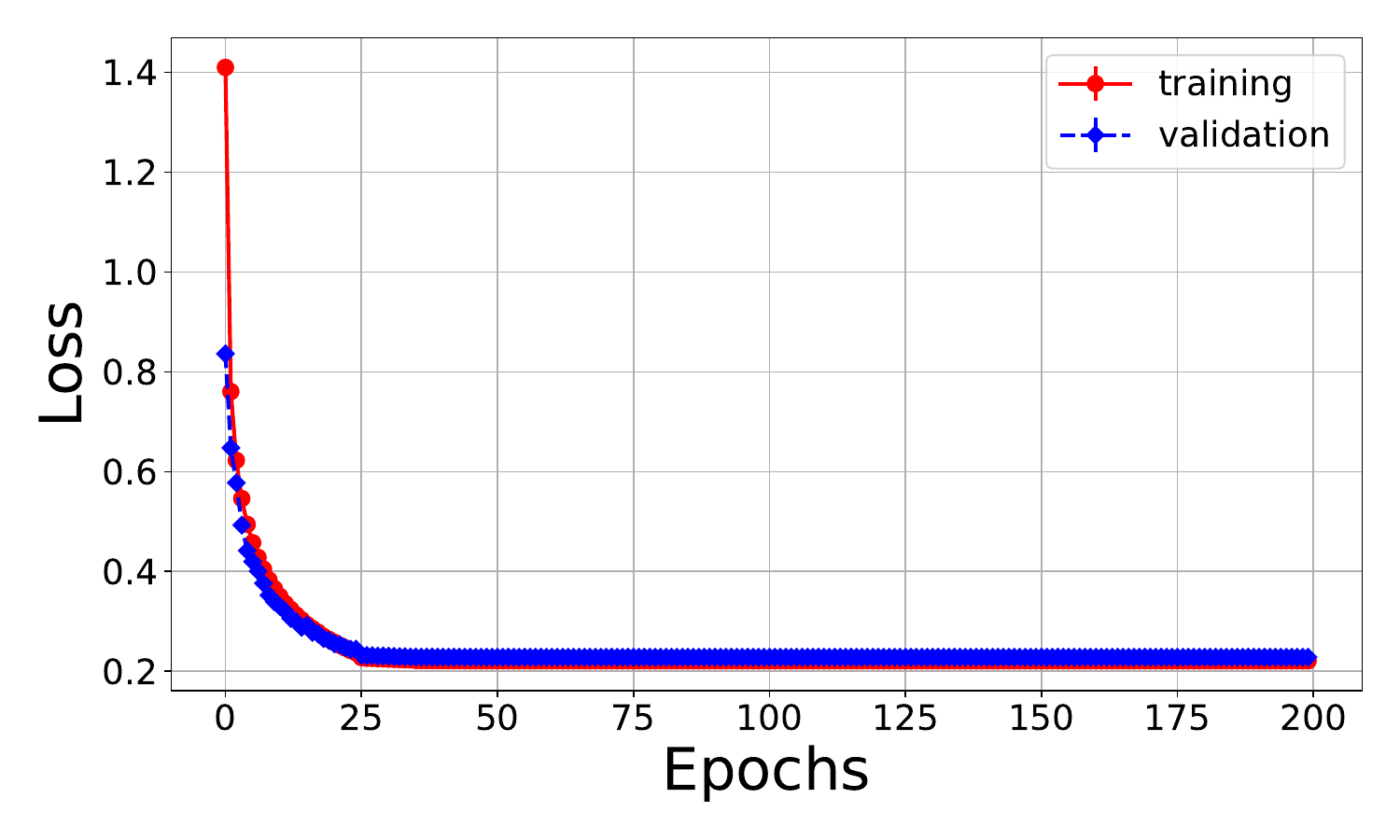}}\label{fig:ablation_early_stop_loss_4}\\
    \subfigure[]{\includegraphics[width=0.49\linewidth]{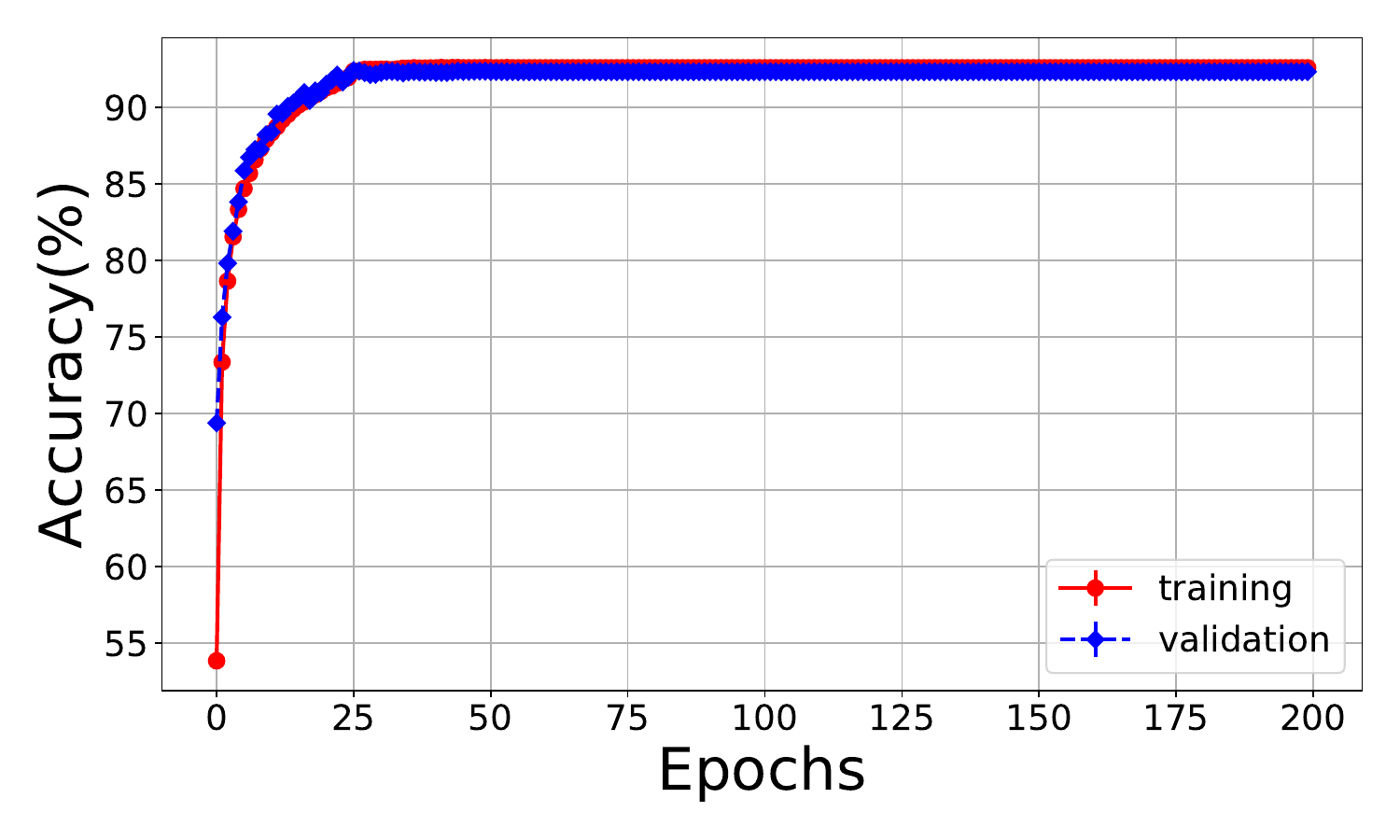}}\label{fig:ablation_early_stop_acc_8}
    \subfigure[]{\includegraphics[width=0.49\linewidth]{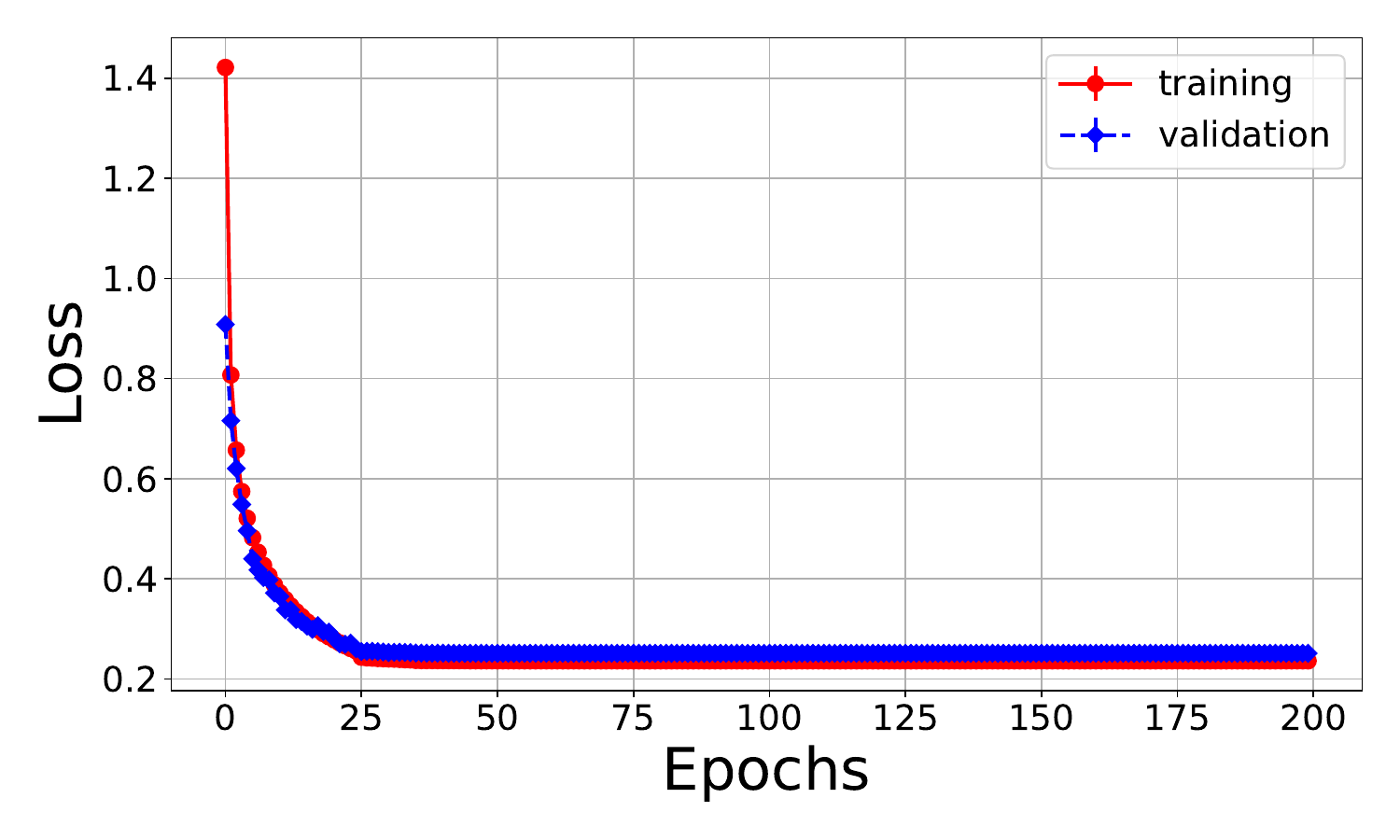}}\label{fig:ablation_early_stop_loss_8}
\caption{Effect of the early-stop of training on (a) the accuracy and (b) the loss of the network with width=32 and depth=4, (c) the accuracy and (d) the loss of the network with width=64 and depth=8.}
\label{fig:ablation_early_stop}
\end{figure}

\subsection{Extension of ~\cref{ssec:Validation_width} to convolutional networks}
\label{ssec:additional_experiments_cnn}

We extend the experiments of ~\cref{ssec:Validation_width} from fully connected networks to convolutional neural networks in~\cref{fig:additional_cnn}. Compared with the fully connected network, the main difference of the convolutional neural network is that the difference between different depths is much larger than fully connected network, which is more in line with the relationship between robustness and depth under He initialization in~\cref{thm:perturbation_stability_lazy}.

\begin{figure}[t]
\centering
    \subfigure[]{\includegraphics[width=0.49\linewidth]{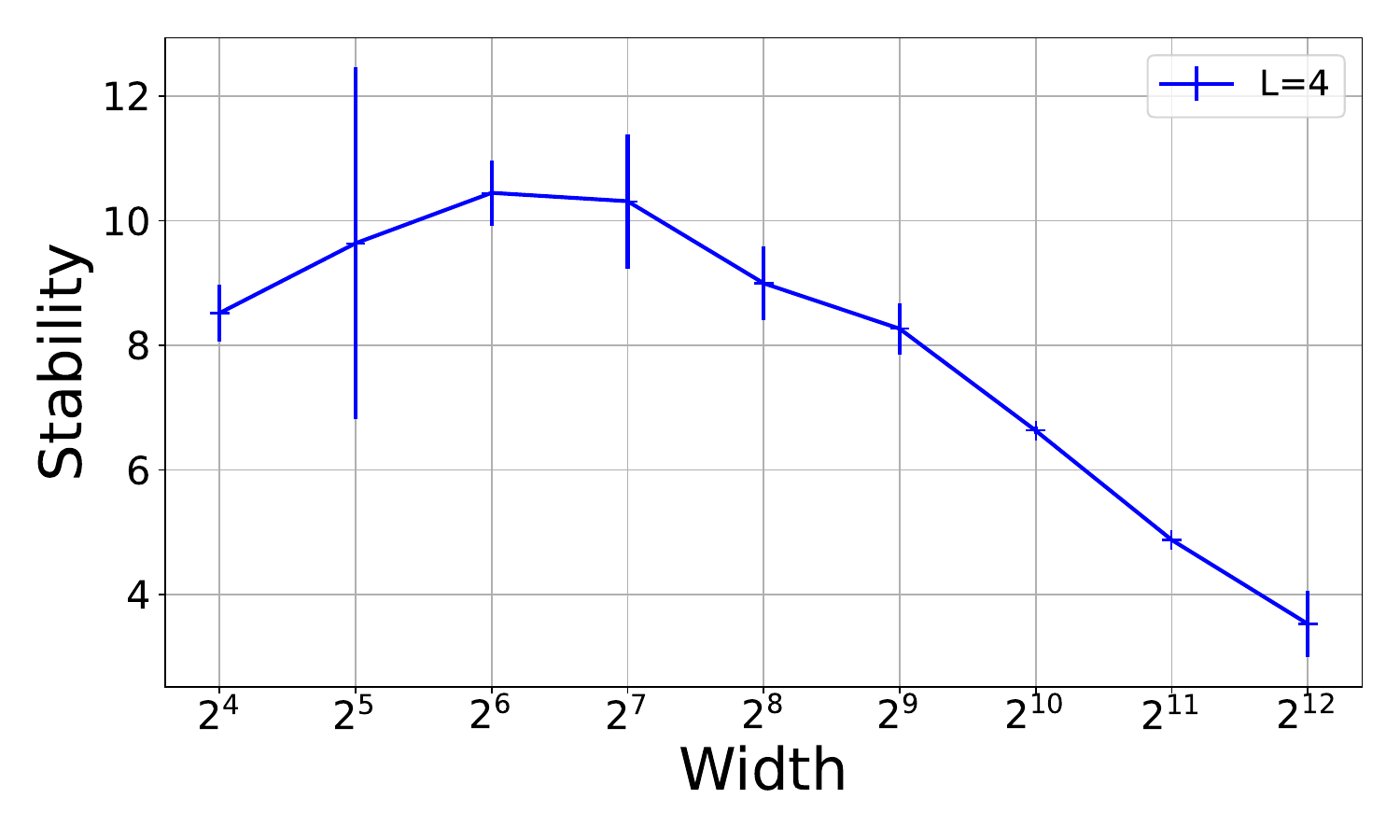}\label{fig:additional_cnn_4}}
    \subfigure[]{\includegraphics[width=0.49\linewidth]{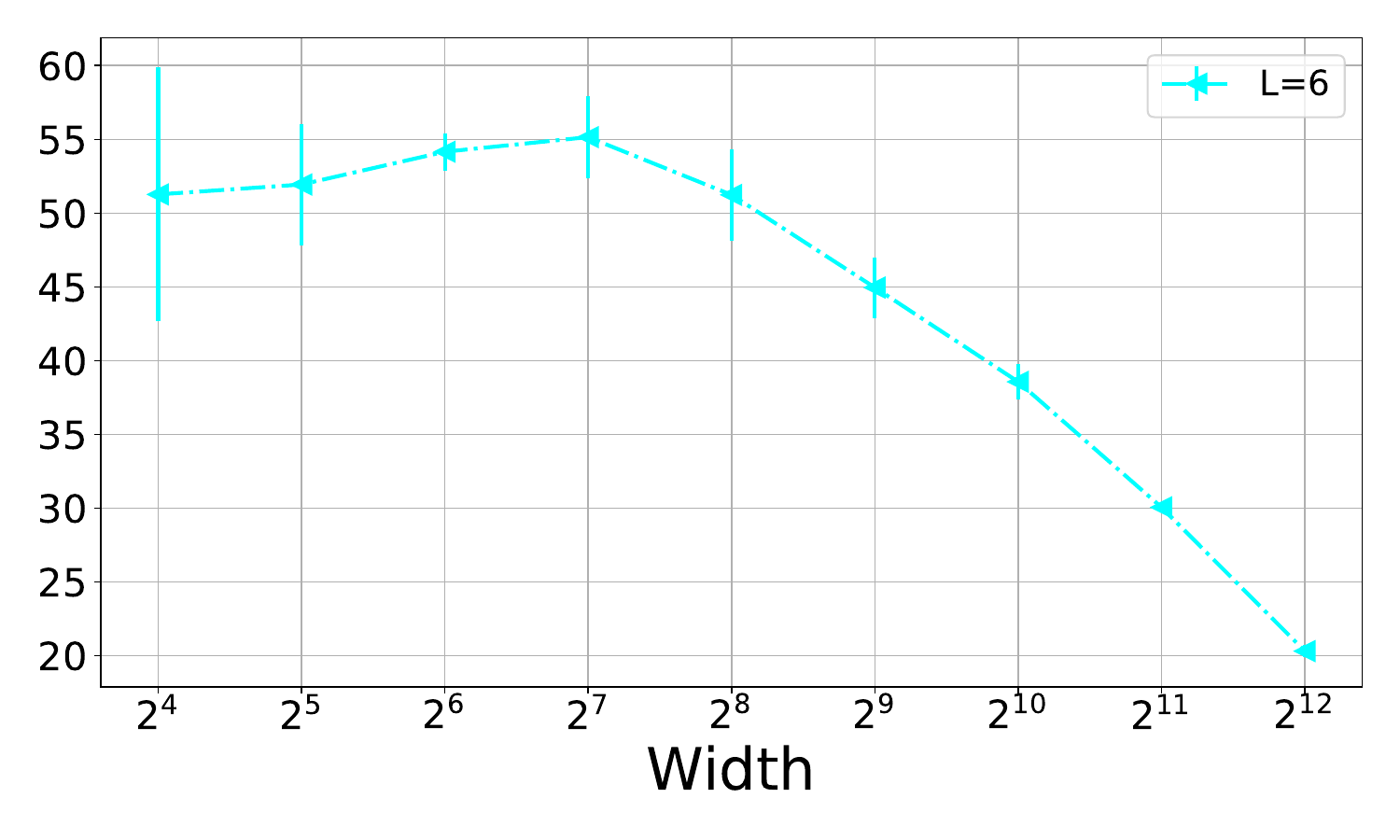}\label{fig:additional_cnn_6}}\\
    \subfigure[]{\includegraphics[width=0.49\linewidth]{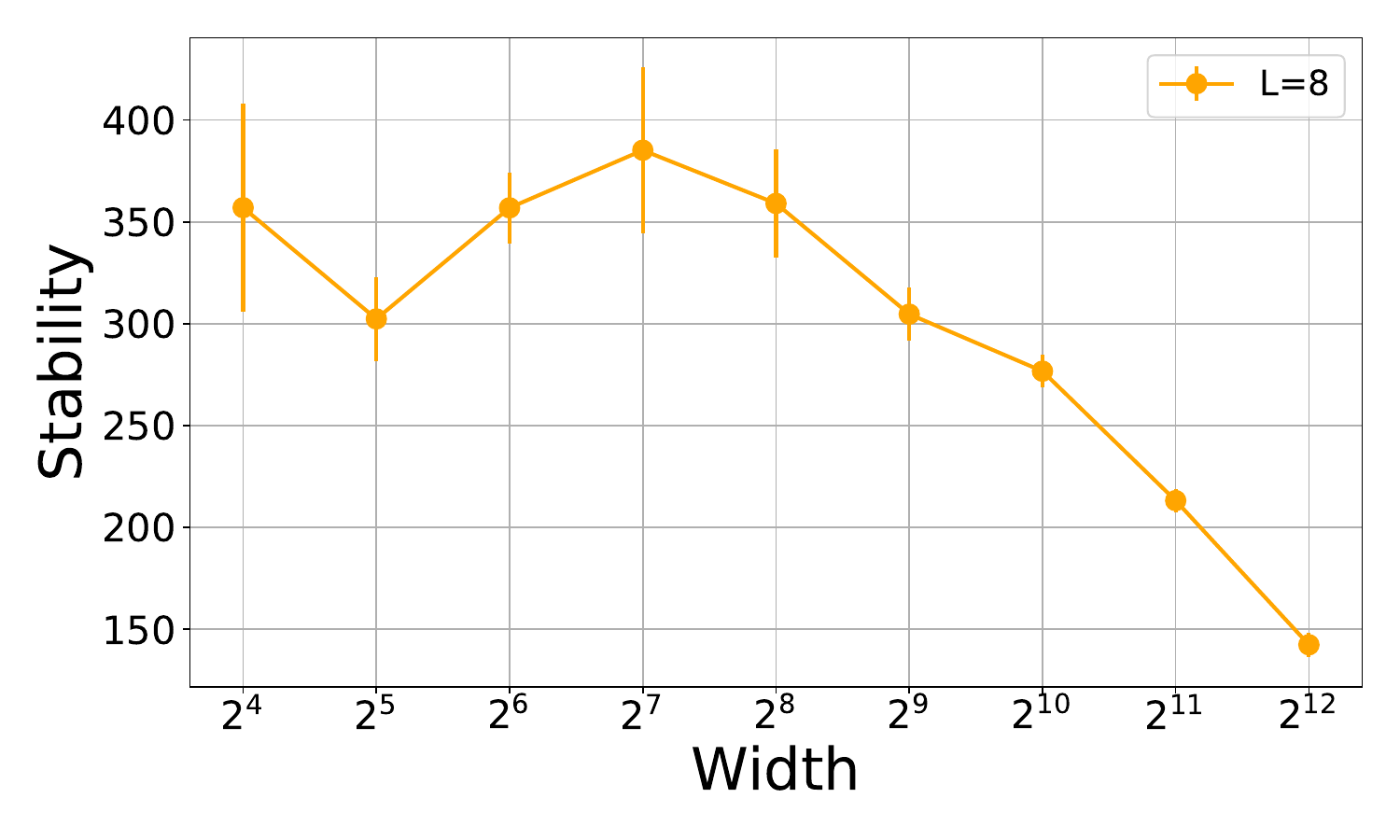}\label{fig:additional_cnn_8}}
    \subfigure[]{\includegraphics[width=0.49\linewidth]{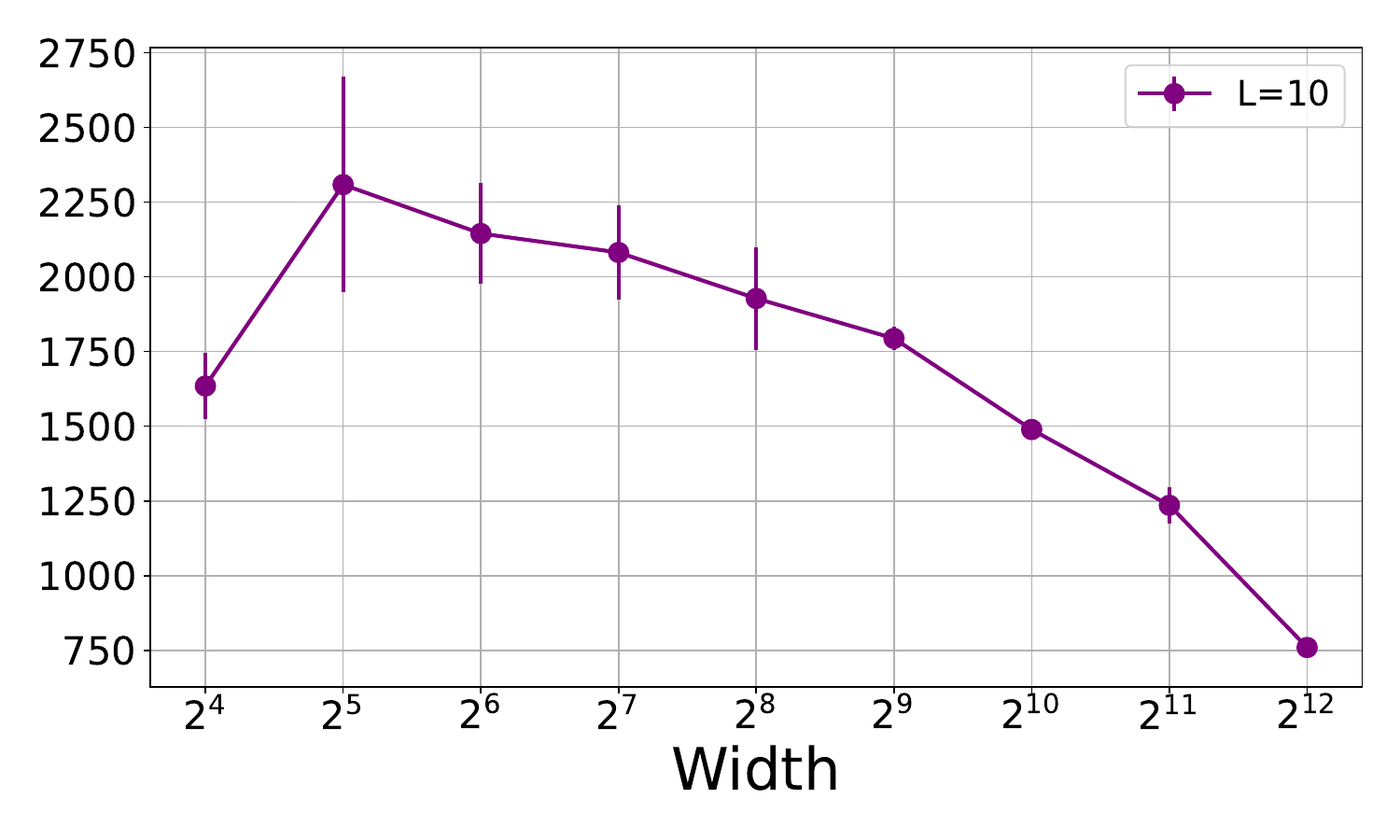}\label{fig:additional_cnn_10}}
\caption{Relationship between the \emph{perturbation stability} and width of CNN under He initialization for different depths of $L=4, 6, 8$ and $10$. The stability values differ substantially across depths, which is more in line with the relationship between robustness and depth under He initialization in~\cref{thm:perturbation_stability_lazy}.}
\label{fig:additional_cnn}
\end{figure}

\subsection{Additional experiments on ResNet}
\label{ssec:additional_experiments_ResNet}

In this section, we extend the experiments in~\cref{ssec:Validation_width} from fully connected networks to ResNet in~\cref{fig:ResNet}. Compared with the fully connected network, the results of ResNet show similar characteristics to our theory on fully connected networks. Specifically, the perturbation stability increases with depth, and an insignificant phase transition can also be seen for width.

\begin{figure}[t]
\centering
    \includegraphics[width=0.49\linewidth]{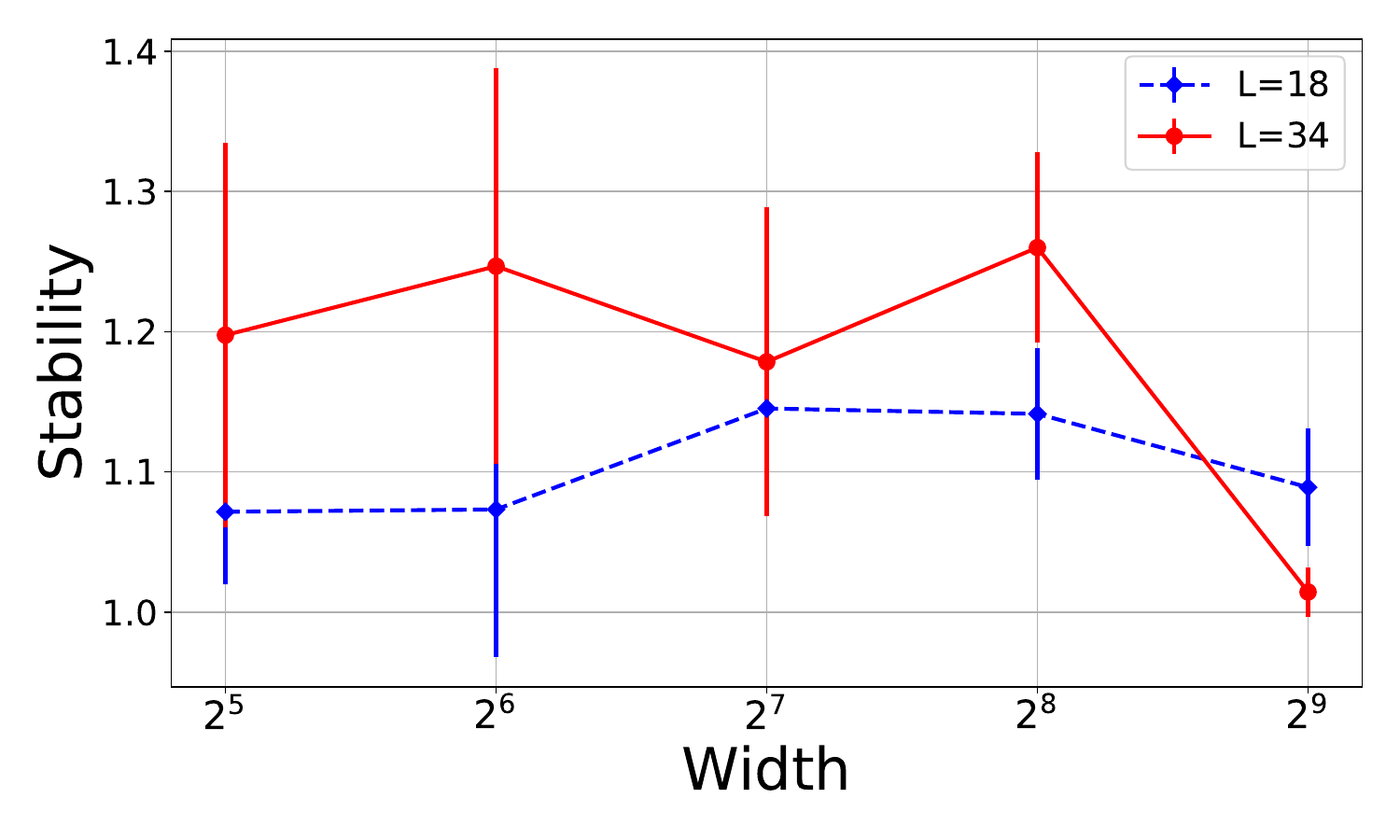}
\caption{Relationship between the \emph{perturbation stability} and width of ResNet-18 and ResNet-34.}\vspace{-0.25cm}
\label{fig:ResNet}
\end{figure}

\subsection{Additional experiments in non-lazy training regime}
\label{ssec:additional_experiments_non_lazy}

We extend the experiments of~\cref{fig:verify_width_non_lazy} to more initializations under non-lazy training regime (the variance of the initial weight are $\frac{1}{m^3}$ and $\frac{1}{m^4}$).~\cref{fig:verify_width_non_lazy_} provides the relationship between robustness and width of neural network for these two initializations and shows that the robustness improves with the increase of the width of network which is consistent with~\cref{thm:perturbation_stability_non_lazy}. However, the difference between different initializations is not as large as our theoretical expectation, which may indicate that the bound in~\cref{thm:perturbation_stability_non_lazy} is not tight enough.

\begin{figure}[t]
\centering
    \subfigure[]{\includegraphics[width=0.49\linewidth]{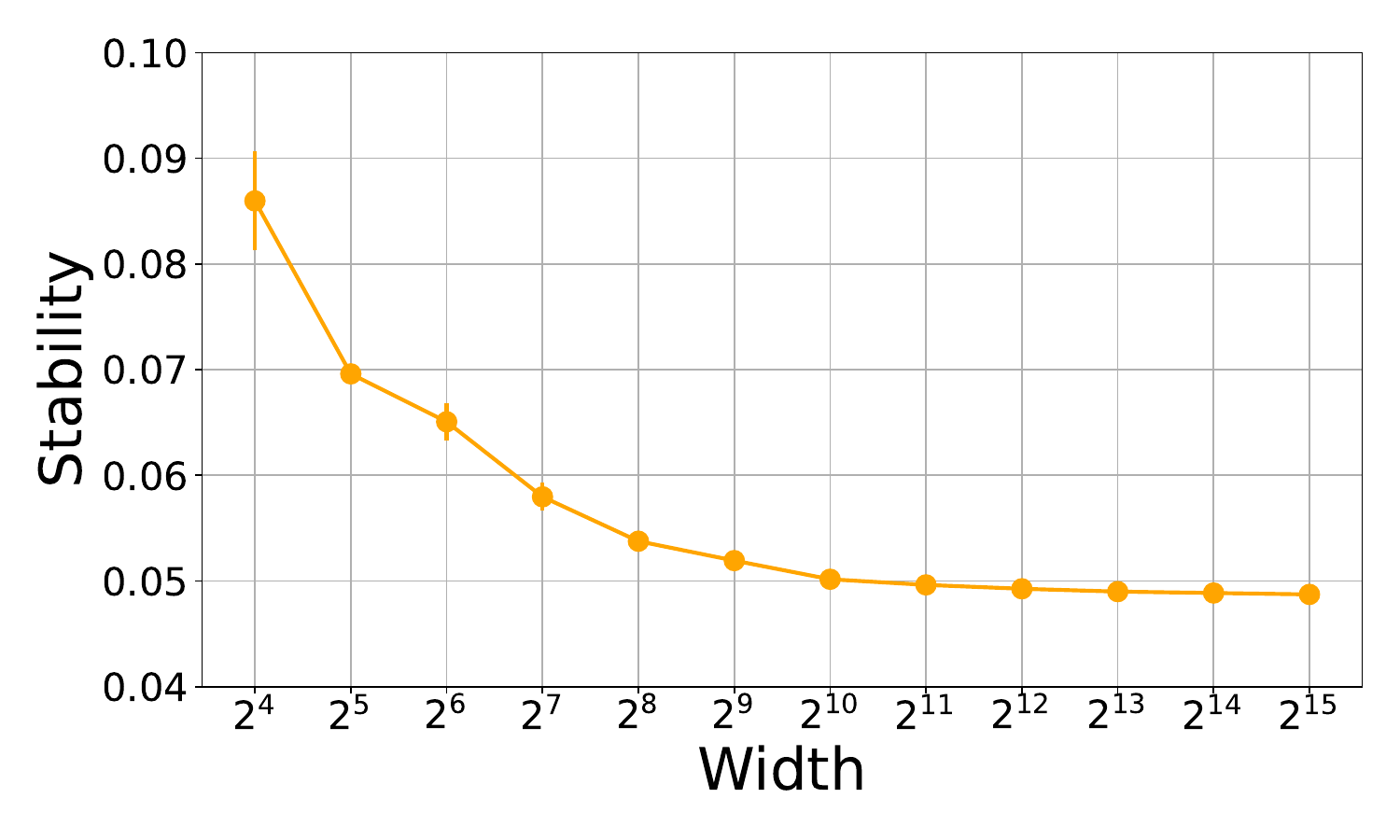}\label{fig:verify_width_non_lazy_1}}
    \subfigure[]{\includegraphics[width=0.49\linewidth]{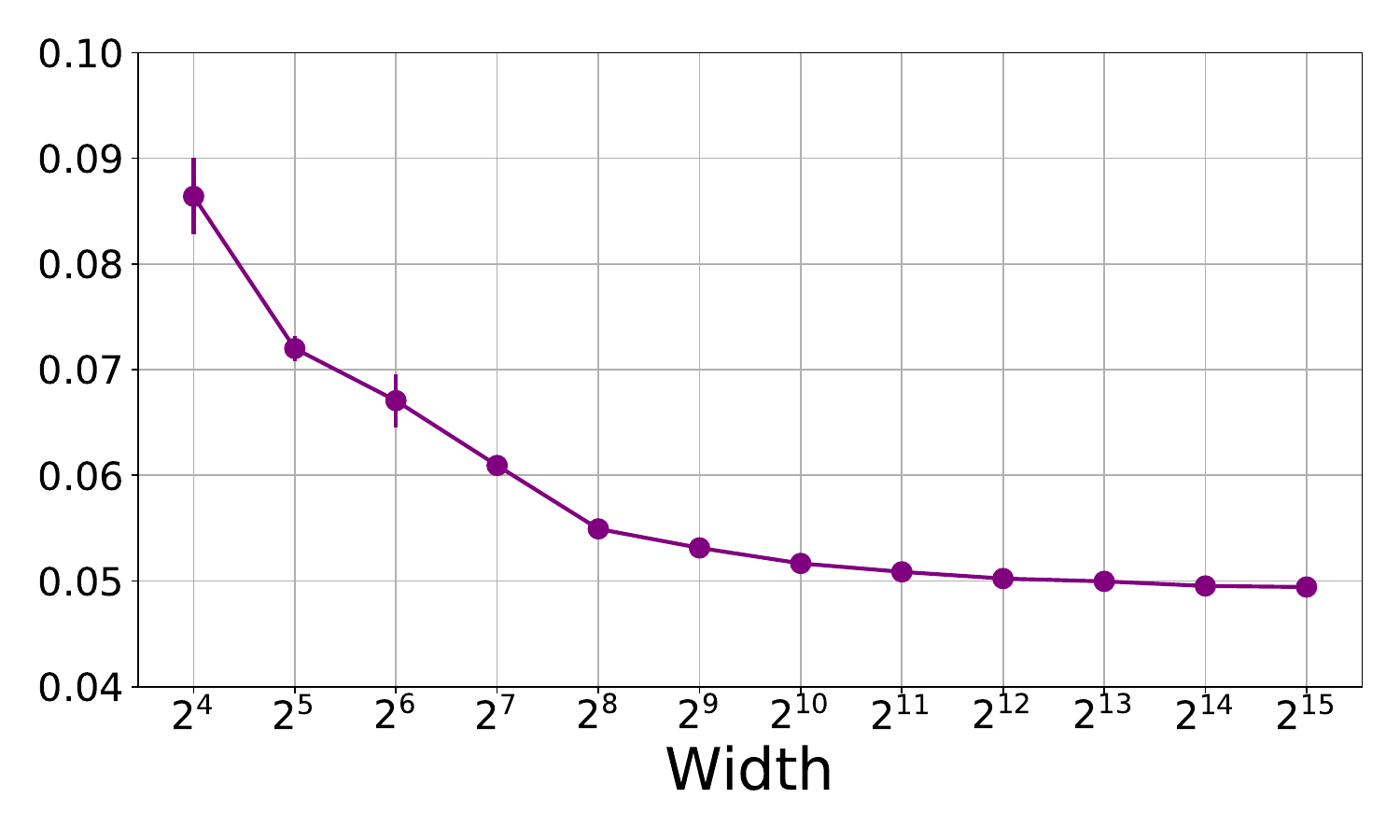}
    \label{fig:verify_width_non_lazy_2}}
\caption{Influence of width of neural network on the perturbation stability under non-lazy training regime. (a) the variance of the initial weight is $\frac{1}{m^3}$. (b) the variance of the initial weight is $\frac{1}{m^4}$.}\vspace{-0.25cm}
\label{fig:verify_width_non_lazy_}
\end{figure}

\subsection{Additional experiments under NTK initialization}
\label{ssec:additional_experiments_NTK}

In this section, we extend the experiments in~\cref{fig:verify_init} from He and LeCun initialization to NTK initialization in~\cref{fig:exp_NTK}. Our experimental results show that, NTK initialization and He initialization yield similar curves, but differ in the curve of  $L=2$.
This may be because the infinite-width NTK is equivalent to the linear model, and the large finite-width network approximates the linear model.
This phenomenon can be more easily detected for two-layer neural networks when compared to deeper networks.

\begin{figure}[t]
\centering
    \includegraphics[width=0.49\linewidth]{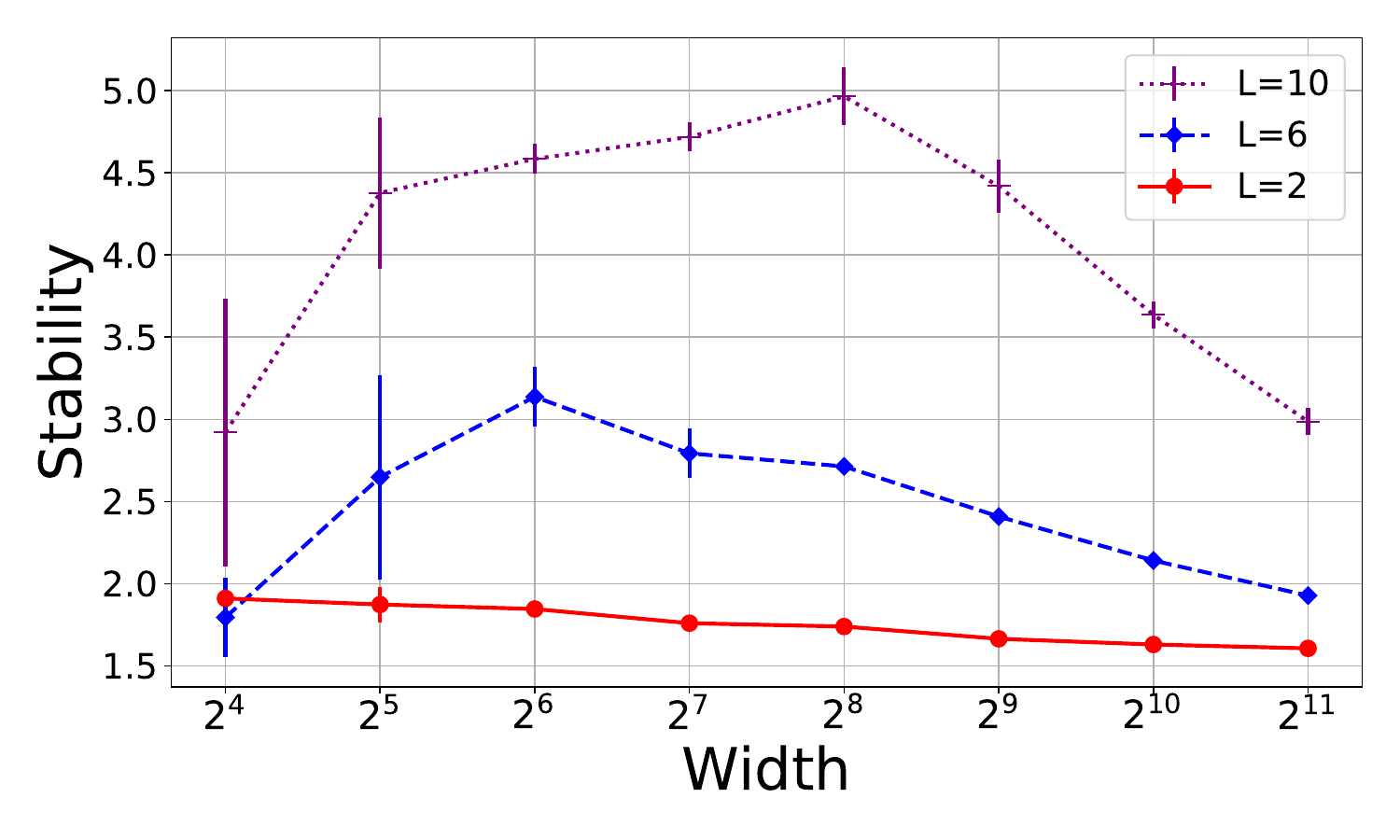}
\caption{Relationship between the \emph{perturbation stability} and depth of FCN under NTK initialization with different depths of $L=2, 6$ and $10$.}\vspace{-0.25cm}
\label{fig:exp_NTK}
\end{figure}

\section{Limitation and discussion}
\label{sec:limitation_and_discussion}

The limitation of this work is mainly manifested in that ~\cref{thm:perturbation_stability_non_lazy} is built on two-layers neural networks. Extending this results to deep neural networks beyond lazy training regime is non-trivial. Firstly, the dynamics of the deep neural network and the bounds of the gap between the initialization and the expectation of the gram matrix will become more complex. Secondly, due to the coupling relationship between different layers, the critical change radius of the weight in~\cref{lemma:bound_changes_of_gram_matrix} is also coupled with each other and is difficult to analyze. Then, due to the superposition of the previous two points, the relationship between the weights changing with time in the early stage of training (similar to~\cref{lemma:bound_w_use_a}) and the width and initialization of the neural network will be difficult to distinguish, which leads to the final result being complex, demanding and difficult to obtain a valid conclusion about width and initialization.

\section{Societal impact}
\label{sec:societal_impact}

This is a theoretical work that explores the interplay of the width, the depth and the initialization of neural networks on their average robustness. Our goal is to obtain an in-depth understanding of the factors that affect the robustness. We do not focus on obtaining any state-of-the-art results in a particular task, which means there are other works that can be used for forming strong adversarial attacks and can be used with malicious intent. 

Despite the theoretical nature of our work, we encourage researchers to further investigate the impact of robustness on the society. We expect robustness to have a key role into a world where neural networks are increasingly deployed into real-world applications. 

\end{document}